%% file: main.tex
\documentclass{article}

\PassOptionsToPackage{numbers}{natbib}
\usepackage[main, final]{neurips_2026}

\usepackage{microtype}
\usepackage{graphicx}
\usepackage{subcaption}
\usepackage{enumitem}
\usepackage{amsmath}
\usepackage{comment}
\usepackage{tabularx}
\usepackage{booktabs} 
\usepackage[T1]{fontenc}
\usepackage{hyperref}
\usepackage{url}
\usepackage{amsfonts}
\usepackage{nicefrac}
\usepackage{xcolor}
\usepackage{amsmath,amssymb,amsthm}
\usepackage{nccmath}
\usepackage{bbm}
\usepackage{mathtools}
\usepackage{float}
\usepackage{algorithm}
\usepackage{algorithmic}
\usepackage{wrapfig}
\usepackage[acronym]{glossaries}
\makeglossaries 
\usepackage{mathrsfs}
\usepackage{longtable}
\usepackage{colortbl}
\definecolor{sectionbg}{gray}{0.92}

\newacronym{marl}{MARL}{Multi-Agent Reinforcement Learning}
\newacronym{dfl}{DFL}{Decision-Focused Learning}
\newacronym{mle}{MLE}{Maximum Likelihood Estimation}
\newacronym{omd}{OMD}{Optimal Model Design}
\newacronym{pomdp}{Dec-POMDP}{Decentralized Partially-Observable Markov Decision Process}

\newcommand{\E}{\mathbb{E}}
\newcommand{\R}{\mathbb{R}}
\newcommand{\N}{\mathcal{N}}
\newcommand{\Sc}{\mathcal{S}}
\newcommand{\Ac}{\mathcal{A}}
\newcommand{\Mc}{\mathcal{M}}
\newcommand{\Lc}{\mathcal{L}}
\newcommand{\bpi}{\boldsymbol{\pi}}
\newcommand{\ba}{\mathbf{a}}

\title{Multi-Agent Decision-Focused Learning via Value-Aware Sequential Communication}

\usepackage[capitalize,noabbrev]{cleveref}

\theoremstyle{plain}
\newtheorem{theorem}{Theorem}[section]
\newtheorem{proposition}[theorem]{Proposition}
\newtheorem{lemma}[theorem]{Lemma}
\newtheorem{corollary}[theorem]{Corollary}
\theoremstyle{definition}
\newtheorem{definition}[theorem]{Definition}
\newtheorem{assumption}[theorem]{Assumption}
\theoremstyle{remark}

\author{%
  Benjamin Amoh \quad Geoffrey Parker \quad Wesley Marrero \\
  Thayer School of Engineering, Dartmouth College \\
  \texttt{benjamin.k.amoh.th@dartmouth.edu}
}

\begin{document}

\maketitle

\begin{abstract}
Multi-agent coordination under partial observability requires sharing complementary private information. Existing methods optimize communication proxies (e.g., mutual information) rather than task performance. We introduce SeqComm-DFL, which unifies sequential communication with decision-focused learning to directly optimize coordination. Our approach features value-aware message generation with sequential Stackelberg conditioning, where messages maximize receiver decision quality and are generated in priority order. The guidance potential is determined by prosocial agent ordering. We extend Optimal Model Design to communication-augmented world models with QMIX factorization for end-to-end training via implicit differentiation. We prove information-theoretic bounds on communication value and establish $\mathcal{O}(1/\sqrt{T})$ convergence for the bilevel optimization. On healthcare and SMAC benchmarks, SeqComm-DFL achieves 4--6$\times$ higher rewards and 13\% win rate improvements, enabling coordination strategies inaccessible under information asymmetry.
\end{abstract}

\section{Introduction}
\label{sec:intro}

Many real-world decision-making problems involve multiple agents operating under partial observability, where each agent observes only a portion of an environment yet must coordinate with others to achieve shared objectives. This setting arises naturally in healthcare resource allocation~\citep{pronovost2002physician}, autonomous multi-robot coordination~\citep{busoniu2008comprehensive}, and traffic control systems~\citep{chu2019multi}. The fundamental challenge is that optimal decisions require information distributed across agents, yet communication bandwidth and latency constraints often preclude sharing all observations directly.

Consider intensive care unit management, where specialists (e.g., cardiologists, pulmonologists, neurologists) observe patient data relevant to their expertise but must coordinate treatment decisions that account for drug interactions and multi-system organ involvement~\citep{vincent2009prevalence}. No single specialist possesses complete information, yet decisions have cascading effects across specialties.

The dominant paradigm for cooperative \gls{marl} is centralized training with decentralized execution \citep{oliehoek2008optimal, kraemer2016multi}, where agents access global information during training but execute using only local observations. While methods like QMIX~\citep{rashid2018qmix}, MAPPO~\citep{yu2022surprising}, and MADDPG~\citep{lowe2017multi} have achieved strong performance, they face fundamental challenges including environment shifts (i.e., each agent's environment changes as others learn) and credit assignment (i.e., determining individual contributions to team rewards). Learned communication offers a promising solution~\citep{foerster2016learning, sukhbaatar2016learning} to environment shifts, but existing protocols typically optimize surrogate objectives such as message quality or mutual information~\citep{wang2020learning}, rather than decision quality.

A parallel challenge arises from the \emph{objective mismatch} problem in model-based reinforcement learning~\citep{lambert2020objective}: world models trained to maximize predictive accuracy may not support optimal policy learning because prediction errors in value-irrelevant dimensions contribute to training loss but not to policy quality. This mismatch extends to learned communication: standard protocols, like the \gls{mle}, optimize for \emph{reconstruction accuracy} rather than \emph{decision quality}, wasting bandwidth on information that does not improve coordination. \Gls{dfl}, also known as predict-and-optimize~\citep{elmachtoub2022smart, tang2024pyepo}, addresses this mismatch by training a predictor end-to-end with a downstream optimizer, backpropagating through the optimization layer to prioritize decision quality over prediction accuracy. However, existing \gls{dfl} methods primarily focus on single-agent settings with \emph{exogenous} uncertainty, where predictions do not affect outcomes.

We propose \textbf{Sequential Communication for Decision-Focused Learning (SeqComm-DFL)}, unifying multi-agent communication learning, model-based reinforcement learning, and \gls{dfl}. Viewed through the predict-and-optimize lens, SeqComm-DFL casts a \emph{communication module} $\phi_\theta$ with parameter $\theta$ as the ``predictor'' and \emph{\gls{marl} policy selection} as the ``optimization task.'' Crucially, our setting involves \emph{endogenous uncertainty}~\citep{tang2024pyepo}: messages are not passive forecasts but active signals that \emph{change} how other agents act, creating feedback loops absent in standard \gls{dfl}. To our knowledge, SeqComm-DFL is the first framework to extend \gls{dfl} to multi-agent systems with endogenous, communication-induced uncertainty.

Our key innovation is \emph{value-aware message generation with sequential Stackelberg conditioning} \citep{zheng2022stackelberg}: messages are optimized to maximize receiver decision quality and generated sequentially in priority order, creating a leader-follower structure. Building on prior work in sequential communication~\citep{ding2024multilevel}, we introduce \emph{guidance potential} for prosocial priority ordering and extend \gls{omd}~\citep{nikishin2022control} to communication-augmented world models with QMIX factorization~\citep{rashid2018qmix}. We provide information-theoretic bounds on communication value (Theorem~\ref{thm:comm_value}) and convergence guarantees (Theorem~\ref{thm:convergence_main}). Experiments on healthcare coordination and original SMAC benchmarks demonstrate rewards $4$--$6$ times higher and win rate gains of over 13\%.

\section{Related Work}
\label{sec:related}

Cooperative \gls{marl} under the centralized training with decentralized execution paradigm~\citep{oliehoek2008optimal, kraemer2016multi} has seen rapid progress. Value decomposition methods, including VDN~\citep{sunehag2018value}, QMIX~\citep{rashid2018qmix}, QTRAN~\citep{son2019qtran}, and QPLEX~\citep{wang2021qplex} factor joint action-values to enable scalable learning, while policy gradient approaches like MAPPO~\citep{yu2022surprising} and MADDPG~\citep{lowe2017multi} extend actor-critic methods to multi-agent settings. Two fundamental challenges persist: non-stationarity \citep{hernandezleal2017survey, papoudakis2019dealing} and credit assignment \citep{foerster2018counterfactual, zhou2020learning}. Opponent modeling~\citep{he2016opponent} and hysteretic Q-learning~\citep{omidshafiei2017deep} aim to mitigate non-stationarity, while COMA~\citep{foerster2018counterfactual} addresses credit assignment via counterfactual baselines. While these methods avoid explicit communication, our work unifies multi‑agent coordination and decision‑focused communication.

Early works like CommNet \citep{sukhbaatar2016learning} and DIAL \citep{foerster2016learning} pioneered end-to-end communication, later enhanced by attention mechanisms \citep{das2019tarmac, niu2021multi}. To improve efficiency, NDQ \citep{wang2020learning} and I2C \citep{ding2020learning} use information-theoretic objectives to minimize overhead. More recently, SeqComm \citep{ding2024multilevel} introduced sequential message priority, and MAIC \citep{yuan2022multi} used messages as incentives for value modulation. However, these methods typically optimize for intermediate surrogates rather than final decision quality. Our work extends SeqComm with value-aware message optimization and integrates communication into decision-focused world model learning.

Model-based reinforcement learning methods like World Models \citep{ha2018world} and Dreamer \citep{hafner2020dreamer} learn dynamics for imagination. To solve the \emph{objective mismatch} problem \citep{lambert2020objective}, value-aware learning \citep{farahmand2017value} and \gls{omd} \citep{nikishin2022control} utilize bilevel optimization. Similarly, \gls{dfl} \citep{donti2017task, elmachtoub2022smart}, or predict-and-optimize \citep{tang2024pyepo}, trains models end-to-end for downstream performance. While recent \gls{dfl} research handles endogenous uncertainty, it remains focused on single-agent settings. \textbf{SeqComm-DFL} is the first to extend \gls{dfl} principles to multi-agent communication with endogenous feedback. See Table \ref{tab:positioning} in Appendix~\ref{app:positioning} for a summary of the positioning of our work. 
\section{Problem Setting}
\label{sec:problem}

We consider a cooperative, \gls{pomdp} $(\N, \Sc, \{\Ac_i\}_{i=1}^N, \{\mathcal{O}_i\}_{i=1}^N, P, R, \gamma)$ \citep{bernstein2002complexity}, where $\N = \{1, \ldots, N\}$ is the set of agents, $\Sc$ is the global state space, $\Ac_i$ is the action space for agent $i \in \N$, $\mathcal{O}_i$ is the local observation space for agent $i$, $P: \Sc \times \Ac \times \Sc \rightarrow [0,1]$ is the transition function, $R: \Sc \times \Ac \mapsto \R$ is the shared reward function, and $\gamma \in [0,1]$ is the discount factor. The set $\Ac = \Ac_1 \times \ldots \times \Ac_N$ denotes the joint action space, and a joint action is denoted $\textbf{a} = (a_1, \ldots, a_N) \in \Ac$, where $a_i \in \Ac_i$ is the action chosen by agent $i$. A communication module function $\phi_\theta:\mathcal{O} \mapsto \Mc$ with parameter $\theta \in \mathbb{R}^{d_\theta}$ generates messages $m_i \in \Mc$ for each agent $i$, where $\Mc \subseteq \mathbb{R}^{d_m}$ is the message space and $d_m$ is the message dimension. We learn a communication-augmented world model $f_\theta: \Sc \times \Ac \times \Mc^N \rightarrow \R \times \Sc$ for predicting rewards and next states conditioned on state, joint actions, and messages. See Appendix~\ref{app:notation} for a summary of our notation.

\section{SeqComm-DFL Framework}
\label{sec:method}

Our framework addresses three interconnected challenges in multi-agent coordination under partial observability: (1) \emph{what to communicate}: generating messages that improve teammates' decisions rather than merely encoding observations; (2) \emph{how to coordinate}: establishing a principled ordering for sequential action selection; and (3) \emph{how to learn}: training world models that optimize for decision quality rather than prediction accuracy. We address these through value-aware message generation, Stackelberg sequential conditioning \citep{zheng2022stackelberg}, and decision-focused bilevel optimization \citep{nikishin2022control}, unified through QMIX-style \citep{rashid2018qmix} value decomposition for scalability. 

\subsection{Communication Architecture with Value-Aware Messaging}
\label{subsec:comm_arch}

Each agent $i \in \N$ first encodes its local observation $o_i \in \mathcal{O}_i$ through a shared message encoder $\phi_\theta$ to produce a base message: $m_i^{\text{base}} = \phi_\theta(o_i)$, where $\phi_\theta$ is a multi-layer perceptron that maps observations to the message space $\Mc$. Unlike standard approaches that optimize messages for surrogate objectives (e.g., reconstruction accuracy), we train the communication module end-to-end with downstream decision quality (a key innovation that directly connects communication to the \gls{dfl} objective).

\subsubsection{The Value-Aware Communication Principle}

Standard communication methods encode observations $o_i$ into messages $m_i$ without considering downstream decision impact. This encoding creates a misalignment: the communication module optimizes for information content (e.g., mutual information $I(m_i; o_i)$ or reconstruction error $\|o_i - \hat{o}_i\|^2$), while the downstream task requires \emph{decision-relevant} information. We address this misalignment by optimizing messages for \emph{receiver decision quality}: Define $\Delta Q_j(m_i) = \max_{a \in \mathcal{A}_j} Q_j(o_j, m_i, a) - \max_{a \in \mathcal{A}_j} Q_j(o_j, \emptyset, a)$ \label{eq:delta_q}, where $\emptyset$ denotes the null (zero-vector) message baseline. We use $Q_j(o_j, m_i, a) = \E[R(o_j, m_i, a)+\gamma \max_{a' \in \Ac_j} Q_j(o', m', a')|o_j, m_i, a]$ to quantify the value for agent $j$ associated with observation $o_j$, message $m_i$ from agent $i$, discount factor $\gamma$, and a transition to observation $o'$ with message $m'$ after taking action $a$. The receiver decision quality quantity $\Delta Q_j(m_i)$ measures how much agent $j$'s best achievable Q-value improves when receiving message $m_i$ from agent $i$. The key insight is that $\Delta Q_j(m_i)$ directly quantifies the \emph{decision value} of communication: information that doesn't improve decisions has $\Delta Q_j = 0$ regardless of its information-theoretic content.

\subsubsection{Value-Aware Training Loss}

We define an auxiliary loss that maximizes the total decision quality improvements across receivers:
\begin{equation}
    \Lc_{\text{VA}}(\theta) = -\frac{1}{B \cdot N(N-1)} \sum_{b=1}^{B} \sum_{i=1}^{N} \sum_{j \neq i} \Delta Q_j(m_i^{(b)}).
\label{eq:va_loss}
\end{equation}
where $b \in \{1,\ldots,B\}$ indexes samples in a mini-batch of size $B$, and $m_i^{(b)}$ is the message generated by agent $i$ for the $b$-th sample. This value-aware loss is relevant because it addresses the core objective mismatch: standard communication training minimizes reconstruction error, but what matters for coordination is whether messages improve teammates' decisions. By directly optimizing $\Delta Q$, we ensure messages encode coordination-critical information rather than observational details. This loss is incorporated into our decision-focused world learning (Subsection \ref{subsec:omd}), ensuring communication optimization is coupled with world model training. 

\subsubsection{Theoretical Connection to Decision-Focused Learning}

Similar to previous \gls{dfl} work \citep{nikishin2022control}, we consider a \emph{true environment loss} $\Lc_{\text{true}}(w; \theta)$ that evaluates how well a critic $Q_w$ with parameter $w \in \R^{|w|}$ trained on model predictions performs on real environment data. By the Envelope Theorem~\citep{bertsekas2012dynamic}, for an optimal critic $w^*$: \(     \frac{\partial \Lc_{\text{true}}(w^*; \theta)}{\partial m_i} = \frac{\partial \Lc_{\text{true}}}{\partial Q_w} \times \frac{\partial Q_w}{\partial m_i} \propto -\sum_{j \neq i} \nabla_{m_i} Q_j(o_j, m_i, a_j^*).\) The gradient of the true loss with respect to messages points toward increasing receiver Q-values, precisely what $\Delta Q_j$ measures. This observation suggests value-aware communication emerges from the decision-focused objective. The value-aware objective aligns communication learning with minimizing true environmental error in \gls{dfl}. Crucially, since $\Delta Q_j$ depends on the critic $Q_w$, which in turn depends on the world model $f_\theta$, this interaction creates an end-to-end pathway from communication to downstream performance.

\subsubsection{Message Refinement Architecture}

We implement value-aware communication through a refinement module $\text{Refine}_\theta: \Mc \times \R \rightarrow \Mc$ that adjusts base messages based on estimated decision impact: \(m_i = m_i^{\text{base}} + \alpha \cdot \text{Refine}_\theta(m_i^{\text{base}}, \Delta \hat{Q_i})\),
where $\Delta \hat{Q}_i = \frac{1}{N-1}\sum_{j \neq i} \Delta \hat{Q}_j(m_i^{\text{base}})$ is the average estimated decision quality improvement across all receivers, computed via a lightweight prediction head trained alongside the critic. To stabilize early training before the critic converges, we use Monte Carlo rollouts to ground $\Delta \hat{Q}$ during warmup and anneal (see Appendix~\ref{app:mc_grounding} for details).

\subsection{Sequential Action Selection with Stackelberg Conditioning}
\label{subsec:stackelberg}

To resolve \emph{relative overgeneralization}~\citep{wei2018multiagent}, where partial observability causes agents to perform sub-optimally due to an inability to anticipate teammates, we adopt a sequential Stackelberg~\citep{vonStackelberg1952} framework. Coordination proceeds in three phases: (1) \emph{negotiation}, computing prosocial guidance potential for priority ordering; (2) \emph{launching}, utilizing sequential conditioning on leaders' committed actions and messages; and (3) \emph{regularization}, using counterfactual influence to ensure causal message impact. Breaking the symmetry of simultaneous selection allows followers to condition on the leader's intentions to reach Pareto-superior equilibria inaccessible under information asymmetry.

\subsubsection{Negotiation Phase: Guidance Potential}

We first determine the priority ordering. Rather than using intention values \citep{ding2024multilevel}, we employ \emph{guidance potential} which measures each agent's capacity to improve \emph{team} outcomes by leading:
\begin{equation}
\label{eq:gp}
\text{GP}_i(s) = \E_{\ba^* \sim \pi^*}\left[Q_{1:N}(s, \ba^* | i^+) - Q_{\text{1:N}}(s, \ba^* | i^-)\right],
\end{equation}
where $Q_{1:N}$ denotes the total joint value across agents $1$ to $N$, $i^+$ indicates that agent $i$ acts as leader (early in sequence), and $i^-$ indicates that agent $i$ acts as follower (later in the sequence). This prosocial criterion allows us to prioritize agents who can most help the team, not those with the strongest individual preferences.

\textbf{Connection to Coordination Information Gap.} The guidance potential is theoretically grounded in the \emph{coordination information gap} $I_i$ (Definition~\ref{def:info_gap} in Appendix~\ref{app:proofs}), which measures private information about team-optimal actions~\citep{bernstein2002complexity}: $\text{GP}_i \propto \sum_{j \neq i} I(M_i; a^*_j \mid o_j) \approx \sum_{j \neq i} \min(I_j, \mathcal{H}(M_i))$, where $\mathcal{H}(M_i)$ is the message entropy. Agents with high $I_i$ (private information critical for teammates) benefit most from leading—their ``strong decisions'' condition the search space for followers. We assign action priorities via Gumbel-softmax $\zeta = \text{argsort}(-\text{GP})$ for differentiable exploration.

\subsubsection{Launching Phase: Sequential Conditioning}
Agents select actions sequentially according to priority ordering $\zeta = (\zeta_1, \ldots, \zeta_N)$, where $\zeta_k$ is the index of the $k$-th agent to act. Agent $\zeta_k$ conditions on \emph{predecessor messages} from higher-priority agents:
$a_{\zeta_k} = \arg\max_{a \in \Ac_{\zeta_k}} Q_{\zeta_k}(o_{\zeta_k}, M_{1:\zeta_k-1}, a)$,\label{eq:stackelberg} where $M_{1:\zeta_k-1} = \{m_{\zeta_j} : j < k\}$ denotes the set of messages from all agents with higher priority than $\zeta_k$. This conditioning creates an explicit leader-follower structure: the first agent (the leader) commits to an action based solely on local information; subsequent agents (followers) observe the leader's messages and adapt accordingly. The Stackelberg structure breaks symmetry and resolves coordination ambiguity that plagues parallel action selection. Algorithm~\ref{alg:seqcomm} summarizes the action selection procedure.

\subsubsection{Regularization Phase: Counterfactual Influence}

To ensure messages have a \emph{causal effect} on decisions rather than merely correlating with them, we add a counterfactual influence loss~\citep{jaques2019social}:
\begin{equation}
    \Lc_{\text{inf}}(\theta) = -\frac{1}{N(N-1)}\sum_{i=1}^{N} \sum_{j \neq i} D_{\text{KL}}\left[\pi_j(\cdot | m_i) \,\|\, \pi_j(\cdot | \emptyset)\right].
    \label{eq:inf_loss}
\end{equation}
This loss encourages messages that \emph{change} the receiver's behavior, as measured by the Kullback–Leibler (KL) divergence~\citep{kullback1951information} between policies with/out communication. Combined with value-aware training, this loss ensures messages are both influential \emph{and} beneficial.

\subsection{Decision-Focused World Model Learning}
\label{subsec:omd}

To resolve \emph{objective mismatch}~\citep{lambert2020objective} we extend \gls{omd}~\citep{nikishin2022control} to multi-agent settings. A bilevel framework optimizes a world model $f_\theta$ for downstream policy regret rather than prediction loss via: (1) an \textbf{inner loop} training a critic on model predictions while evaluating on real data in an \textbf{outer loop}; (2) a \textbf{coordination-aware objective} that prevents message apathy by incentivizing the critic to distinguish informative signals; and (3) \textbf{hypergradient computation} utilizing the Implicit Function Theorem (IFT)~\citep{krantz2002implicit} and an efficient Conjugate Gradient (CG) solver~\citep{hestenes1952methods}.

\subsubsection{Bilevel Optimization Formulation}

The core insight of our bilevel optimization is to merge world model accuracy and decision quality into a single objective: rather than training the model to predict well (inner loop), we train the model parameters to directly optimize performance when the resulting critic is evaluated on environment data (outer loop). The outer problem optimizes model parameters $\theta$ for true environment performance, while the inner problem trains the critic $Q_w$ using model predictions:
\begin{subequations}
\label{eq:bilevel_optimization}
\begin{align}
    \text{Outer:} \quad & \min_\theta \; \mathcal{L}_{\text{true}}(w^*(\theta); \theta) \label{eq:outer} \\
    \text{Inner:} \quad & \text{s.t.} \quad w^*(\theta) = \arg\min_w \mathcal{L}_{\text{model}}(w; \theta) \label{eq:inner}
\end{align}
\end{subequations}
This formulation captures that the optimal critic $w^*$ depends on the world model $\theta$, so updating $\theta$ affects the world model predictions \emph{and} the critic that evaluates them.

\subsubsection{Inner Loop: Model-Based Critic Training}

The inner problem trains the critic $Q_w$ to satisfy the Bellman equation under the \emph{learned} world model dynamics, running $K_{\text{inner}}$ gradient descent steps to approximately solve the inner optimization. Given a dataset $\mathcal{D} = \{(s^{(b)}, \ba^{(b)})\}_{b=1}^{B}$ of $B$ batches of state-action pairs, the model loss is:
\begin{equation}
\label{eq:model_loss}
\mathcal{L}_{\text{model}}(w; \theta) = \frac{1}{B}\sum_{b=1}^{B} \left( Q_w(s^{(b)}, \mathbf{a}^{(b)}, M^{(b)}) - y_{\text{model}}^{(b)} \right)^2 + \lambda_{\text{reg}} \|w\|_2^2
\end{equation}
where the target incorporates world model predictions: \(y_{\text{model}} = \hat{r} + \gamma V_{\bar{w}}(\hat{s}', M'), (\hat{r}, \hat{s}') = f_\theta(s, \ba, M), M' = \phi_\theta(\hat{s}')\). Here, $\bar{w}$ denotes target network parameters updated via exponential moving average~\citep{polyak1992acceleration}, $M'$ is the message tensor re-computed for the predicted next state, and $V_{\bar{w}}$ is the soft value function computed by enumerating joint actions before QMIX factorization is learned:
\(V_{\bar{w}}(s, M) = \tau \cdot \log \sum_{\ba \in \Ac^N} \exp\left(\frac{Q_{\bar{w}}(s, \ba, M)}{\tau}\right),\)
where $\tau \in \mathbb{R}_{>0}$ is a temperature parameter. The inner loop runs $K_{\text{inner}}>0$ gradient steps to approximately solve~\eqref{eq:inner}, yielding $w^{(K_{\text{inner}})} \approx w^*(\theta)$.

\paragraph{Coordination-Aware Inner Loop: Preventing Message Apathy.}

In bilevel communication, a failure mode termed \emph{inner-loop apathy} occurs when the critic $Q_w$ ignores messages and blocks the optimization and updates of communication parameters. To ensure meaningful optimization updates, we augment the inner-loop with a \emph{message-awareness} term \citep{yuan2022multi}:

\begin{equation}
\label{eq:aware_loss}
\begin{split}
    \mathcal{L}_{\text{aware}}(w) = \frac{1}{B} \sum_{b=1}^{B} \max\Big(0, \; \epsilon_{\text{margin}} - \big|Q_w(s^{(b)}, \mathbf{a}^{(b)}, M^{(b)}) - Q_w(s^{(b)}, \mathbf{a}^{(b)}, \emptyset)\big|\Big),
\end{split}
\end{equation}
where $\emptyset$ denotes the null messages (consistent with Section~\ref{subsec:comm_arch}) and $\epsilon_{\text{margin}} > 0$ is a margin hyperparameter. This hinge loss penalizes the critic whenever it assigns nearly equal Q-values to an observed message $M^{(b)}$ and the null message $\emptyset$, i.e, when  $Q(s, a, M^{(b)}) \approx Q(s, a, \emptyset)$, treating messages as ``incentives'' that modulate the Q-values. The complete inner-loop objective becomes: $\Lc_{\text{inner}}(w; \theta) = \Lc_{\text{model}}(w; \theta) + \lambda_{\text{aware}} \Lc_{\text{aware}}(w)$.

\subsubsection{Outer Loop: Decision-Focused World Model Update}

The outer problem evaluates the critic on \emph{true environment} data to assess its generalization. Given an environment dataset $\mathcal{D}_{\text{env}} = \{(s^{(b)}, \ba^{(b)}, r^{(b)}, s'^{(b)})\}_{b=1}^{B}$, we define the environment true loss as: 
\begin{equation*}
    \Lc_{\text{true}}(w; \theta) = \frac{1}{B}\sum_{b=1}^{B} \left( Q_w(s^{(b)}, \ba^{(b)}, M^{(b)}) - y_{\text{true}}^{(b)} \right)^2,
\end{equation*}
where $y_{\text{true}} = r + \gamma V_{\bar{w}}(s', M'_{\text{true}})$ uses \emph{observed} rewards and next states. While $\Lc_{\text{model}}$ in the inner loop uses $(\hat{r}, \hat{s}')$ from the world model, $\Lc_{\text{true}}$ uses $(r, s')$ from the environment.

\paragraph{Hypergradient Computation via Implicit Differentiation.}

Since $w^*$ depends on $\theta$, the outer objective's total derivative requires differentiating \emph{through} the inner optimization. By the chain rule:
\begin{equation}
\label{eq:total_deriv}
\frac{d\mathcal{L}_{\text{true}}}{d\theta} = \underbrace{\nabla_\theta \mathcal{L}_{\text{true}}(w^*, \theta)}_{\text{direct effect}} + \underbrace{\left(\frac{dw^*}{d\theta}\right)^\top \nabla_w \mathcal{L}_{\text{true}}(w^*, \theta)}_{\text{indirect effect}}.
\end{equation}
The Jacobian $\frac{dw^*}{d\theta} \in \R^{|w| \times |\theta|}$ is intractable to compute directly. We apply the IFT~\citep{krantz2002implicit} to the Bellman residual's first-order optimality condition. At the inner-loop fixed point, $\nabla_w \Lc_{\text{model}}(w^*; \theta) = 0$. The IFT enables us to differentiate this optimality condition with respect to $\theta$, implicitly characterizing how $w^*$ changes as $\theta$ varies without explicitly solving for $w^*(\theta)$. Differentiating implicitly: \(\frac{dw^*}{d\theta} = -\left[\nabla^2_{ww} \Lc_{\text{model}}(w^*, \theta)\right]^{-1} \nabla^2_{\theta w} \Lc_{\text{model}}(w^*, \theta)\). This IFT application is the key insight from \gls{omd}~\citep{nikishin2022control}: rather than backpropagating through $K_{\text{inner}}$ gradient steps, we treat the converged inner loop as an implicit function and differentiate at equilibrium, decoupling the forward solve from the backward pass (see Appendix~\ref{app:omd}). Substituting $\frac{dw^*}{d\theta}$ into~\eqref{eq:total_deriv} yields the hypergradient:
\begin{equation}
\label{eq:hypergradient}
    \frac{d\Lc_{\text{true}}}{d\theta} = \nabla_\theta \Lc_{\text{true}} - \underbrace{(\nabla^2_{\theta w} \Lc_{\text{model}})^\top}_{\text{mixed Hessian}} \underbrace{(\nabla^2_{ww} \Lc_{\text{model}})^{-1}}_{\text{inverse Hessian } H^{-1}} \underbrace{\nabla_w \Lc_{\text{true}}}_{b}.
\end{equation}


This hypergradient updates $\theta$ to minimize true environment loss while accounting for how the critic parameters $w^*$ implicitly depend on $\theta$. However, computing $H^{-1}b$ directly is intractable for high-dimensional $w$, requiring efficient approximation.

\paragraph{Efficient Computation via Conjugate Gradient.}
The product $H^{-1}b$ in~\eqref{eq:hypergradient} is computed without forming the full Hessian. We solve the linear system $(H + \lambda I)v^* = b$ using CG~\citep{hestenes1952methods,shewchuk1994introduction}, where $H = \nabla^2_{ww} \Lc_{\text{model}}$ is the Hessian of the model loss with respect to the critic parameters, $I$ is the identity matrix, $b = \nabla_w \Lc_{\text{true}}$ is the gradient of the true loss with respect to the critic parameters, and \textbf{$\lambda > 0$} is a damping coefficient for numerical stability.
Each CG iteration requires one Hessian-vector product, computed efficiently via two autodiff passes: $Hv = \nabla_w(\nabla_w \Lc_{\text{model}}^\top v)$. The final hypergradient  $\nabla_\theta \Lc_{\text{true}} - \nabla_\theta(\nabla_w \Lc_{\text{model}}^\top v^*)$ is then computed via a single backward pass through the mixed Hessian term.

\subsection{Training Procedure}
\label{subsec:training}

Algorithms~\ref{alg:seqcomm} and~\ref{alg:main} present the action selection and training procedures, respectively. We employ a warmup schedule that gradually increases the SeqComm-guided exploration, allowing the world model to stabilize before agents rely on learned communication.

\begin{figure}[t]
\begin{minipage}[t]{0.48\textwidth}
\begin{algorithm}[H]
\caption{Value-Aware Action Selection}
\label{alg:seqcomm}
\begin{algorithmic}[1]
\REQUIRE State $s$, comm. module $\phi_\theta$, critic $Q_w$, refinement net $\text{Refine}_\theta$
\STATE \textbf{Negotiation Phase}: Compute $\text{GP}_i$ for all agents $i$ using Eq. \ref{eq:gp}
\STATE Determine priority ordering: $\zeta \gets \text{argsort}(-\text{GP})$ via Gumbel-softmax
\STATE \textbf{Message Generation Phase}:
\FOR{each agent $i \in \{1, \ldots, N\}$}
    \STATE Encode observation: $m_i^{\text{base}} \gets \phi_\theta(o_i)$
    \STATE Estimate decision impact (Eq.~\ref{eq:delta_q}):\\ $\Delta \hat{Q}_i \gets \frac{1}{N-1}\sum_{j \neq i} \Delta Q_j(m_i^{\text{base}})$
    \STATE Refine message: $m_i \gets m_i^{\text{base}} + \alpha \cdot \text{Refine}_\theta(m_i^{\text{base}}, \Delta \hat{Q}_i)$
\ENDFOR
\STATE \textbf{Launching Phase (Stackelberg)}:
\FOR{$k = 1$ to $N$}
    \STATE Agent $\zeta_k$ selects: $a_{\zeta_k} \gets \arg\max_a Q_{\zeta_k}(o_{\zeta_k}, M_{1:k-1}, a)$
\ENDFOR
\STATE \textbf{Return} Joint action $\ba = (a_1, \ldots, a_N)$
\end{algorithmic}
\end{algorithm}
\end{minipage}%
\hfill
\begin{minipage}[t]{0.48\textwidth}
\begin{algorithm}[H]
\caption{SeqComm-DFL Training}
\label{alg:main}
\begin{algorithmic}[1]
\REQUIRE Environment $\mathcal{E}$, iterations $T$, warmup $T_w$, weights $\lambda_{\text{VA}}, \lambda_{\text{inf}}, \lambda_{\text{aware}}$
\STATE Init world model $f_\theta$, comm. module $\phi_\theta$, critic $Q_w$, target $Q_{\bar{w}}$
\STATE Set annealing schedule: $\beta_t \gets \min(t/T_w, 1)$
\FOR{$t = 1$ to $T$}
    \STATE Collect experience $\mathcal{D}, \mathcal{D}_{\text{env}}$ from environment $\mathcal{E}$
    \STATE Compute hybrid $\Delta Q$: $\Delta \hat{Q} \gets (1-\beta_t)\Delta Q^{\text{MC}} + \beta_t \Delta Q^{\text{critic}}$
    \STATE \textbf{Inner Loop (Eq.~\ref{eq:model_loss}, \ref{eq:aware_loss})}: For $K_{\text{inner}}$ steps,\\ $w \gets \arg\min_w \Lc_{\text{model}}(w;\theta) + \lambda_{\text{aware}}\Lc_{\text{aware}}(w)$
    \STATE \textbf{Outer Loop (Eq.~\ref{eq:va_loss}, \ref{eq:inf_loss}, \ref{eq:hypergradient})}: \\ $\theta \gets \theta - \eta \left( \frac{d\Lc_{\text{true}}}{d\theta} + \lambda_{\text{VA}} \nabla \Lc_{\text{VA}} + \lambda_{\text{inf}} \nabla \Lc_{\text{inf}} \right)$
    \STATE Update target network: $\bar{w} \gets \tau \bar{w} + (1-\tau) w$
\ENDFOR
\STATE \textbf{Return} $f_\theta, Q_w$
\end{algorithmic}
\end{algorithm}
\end{minipage}
\end{figure}

\section{Theoretical Analysis}
\label{sec:theory}

We analyze SeqComm-DFL in the Dec-POMDP setting $\{\N, \Sc, \{\Ac_i\}, \{\mathcal{O}_i\}, P, R, \gamma \} $ where agents receive local observations but share a common reward. We establish formal foundations for SeqComm-DFL by quantifying the necessity of communication, the advantage of decision-focused world models, and bilevel stability. Full proofs and regularity conditions are provided in Appendices \ref{app:omd} and \ref{app:proofs}.

\paragraph{Information-Theoretic Necessity.} Partial observability creates a coordination information gap $I_i = \mathcal{I}(s; a_i^* | o_i)$, representing the optimal action information contained in the state $s$ that is inaccessible from local observation $o_i$. We present a bound on this gap in the following theorem: 

\begin{theorem}[Communication Lower Bound]
\label{thm:comm_value}
Let $\boldsymbol{\pi}^{\text{comm}}$ be a joint policy utilizing SeqComm messages, and $\boldsymbol{\pi}^{\text{no-comm}}$ be a policy restricted to local observations. For a Dec-POMDP with per-action reward gap $\Delta_{\min}$ \citep{bertsekas2012dynamic} and discount factor $\gamma$, the expected performance gain from communication satisfies:
\(\mathbb{E}[V^{\boldsymbol{\pi}^{\text{comm}}}] - \mathbb{E}[V^{\boldsymbol{\pi}^{\text{no-comm}}}] \geq \frac{\Delta_{\min}}{1-\gamma} \sum_{i=1}^N \frac{(I_i-\log 2)_+}{\log|\Ac_i|},\)
where $(x)_+=\max(x,0)$ and $I_i=\mathcal{I}(s;a_i^*\mid o_i)$ is the coordination information gap.
\end{theorem}
This bound proves that communication is necessary when the Fano term is positive. The lower bound grows with hidden optimal-action information: when local observations omit decision-critical state, communication provides greater benefit. SeqComm's sequential launch phase reduces this gap toward zero by sharing actual actions, enabling followers to reach a Pareto-superior Stackelberg equilibrium.

\paragraph{Decision-Focused Advantage.} Let $\hat{P}$ and $\hat{R}$ be the learned transition and reward models using \gls{mle}, with errors $\epsilon_P = \sup_{s,\mathbf{a}} \|P(\cdot|s,\mathbf{a}) - \hat{P}(\cdot|s,\mathbf{a})\|_1$ and $\epsilon_R = \sup_{s,\mathbf{a}} |R(s,\mathbf{a}) - \hat{R}(s,\mathbf{a})|$. Moreover, let $\epsilon = \sup_{s,\mathbf{a}} |(\mathcal{B} Q)(s,\mathbf{a}) - Q(s,\mathbf{a})|$ denote the Bellman error~\citep{sutton2018reinforcement}, where $(\mathcal{B} Q)(s,\mathbf{a}) = \mathbb{E}[r + \gamma \max_{\mathbf{a}'} Q(s', \mathbf{a}')]$. By minimizing the Bellman error via IFT, SeqComm-DFL's bilevel formulation yields a tighter $Q^*$ approximation than \gls{mle}, as shown in the following proposition:
\begin{proposition}[$Q^*$ Approximation Error]
\label{prop:q_error}
Let $\hat{Q}_{\text{MLE}}$ and $\hat{Q}_{\text{DFL}}$ be fixed points of Bellman operators induced by \gls{mle} and SeqComm-DFL. With dynamics error $\epsilon_P $, reward error $\epsilon_R$, and Bellman error $\epsilon$, the DFL error is tighter: $\|Q^* - \hat{Q}_{\text{MLE}}\|_\infty \leq \frac{\epsilon_R}{1-\gamma} + \frac{\gamma \epsilon_P r_{\max}}{(1-\gamma)^2}$ and $\|Q^* - \hat{Q}_{\text{DFL}}\|_\infty \leq \frac{\epsilon}{1-\gamma}$.
\end{proposition}
By directly optimizing the Bellman error, SeqComm-DFL's world model ignores value-irrelevant prediction errors that would otherwise degrade performance in decoupled learning.

\paragraph{Bilevel Convergence.} Despite the complexity of differentiating through the inner-loop critic ($\epsilon_{\text{inner}}$ approximation error) and CG solver ($\epsilon_{\text{CG}}$ linearization error), the total bias $\epsilon_{\text{bias}} = \epsilon_{\text{inner}} + \epsilon_{\text{CG}}$ decays at $O(1/\sqrt{T})$~\citep{ghadimi2018approximation}, yielding a standard first-order convergence:

\begin{theorem}[Convergence to Stationary Points]
\label{thm:convergence_main}
Under the regularity conditions in Appendix~\ref{app:proofs}, an outer learning rate $\eta_\theta = O(1/\sqrt{T})$, and maintaining hypergradient bias $\epsilon_{\text{bias}} = O(1/\sqrt{T})$ by setting $K_{\text{inner}}, K_{\text{CG}} = O(\log T)$, the average squared gradient of the true loss satisfies: \(\frac{1}{T} \sum_{t=1}^T \mathbb{E} \|\nabla_\theta \mathcal{L}_{\text{true}}(\theta^{(t)})\|^2 \leq O\left(\frac{1}{\sqrt{T}}\right) \).
\end{theorem}

This theorem shows SeqComm-DFL converges to a stationary point of the true environment loss at the standard $O(1/\sqrt{T})$ rate despite the bilevel structure and the endogenous uncertainty introduced by communication. In practice this means that running $T$ outer iterations with logarithmically growing inner and CG budgets is sufficient to drive the gradient of the decision-focused objective to zero.

\section{Experiments}
\label{sec:experiments}

We evaluate SeqComm-DFL on: (1) a hospital environment with information asymmetry, and (2) Multi-Agent Particle Environment (MPE) and (3) the StarCraft Multi-Agent Challenge (SMAC). All main experiments use 10 random seeds. Appendices \ref{app:hyperparams}, \ref{app:hospital}, and \ref{app:ablation} include our hyperparameters, environment details, and ablation studies, respectively.

\subsection{Hospital Environment: A Dec-POMDP with Structured Information Asymmetry}
\label{subsec:hospital}

We design a collaborative hospital environment that instantiates a Dec-POMDP with \emph{genuine} partial observability, where communication provides measurable value.

\subsubsection{Structured Partial Observability}

The environment has $N=3$ specialist physician agents (e.g., cardiology, pulmonology, neurology) managing $\mathcal{P}=100$ patients. Each physician $i$ observes only specialty-gated risk factors for patient $j$: $o_i = h_{j,\sigma_i}$ if specialty $\sigma_i$ matches the patient condition, else $o_i = 0$. These observations create \emph{fundamental information asymmetry}—a cardiologist sees cardiac risk $h_{j,1}$ but not neurological risk $h_{j,3}$, making multi-condition treatment decisions impossible without inter-specialist communication. Each physician selects treatment intensity $a_i \in \{0,1,2\}$(low, medium, high). The reward incentivizes effective treatment (vitals improvement) while penalizing misdiagnosis, adverse drug interactions, and excess resource use. Severity is the $\ell_1$ deviation from healthy baseline; improvement is the per-timestep reduction averaged across patients. Full environment details are in Appendix~\ref{app:hospital}.

\subsubsection{Coordination-Dependent Dynamics}

The specialty-gating and drug interaction dynamics place this environment in the class of \textit{transition-coupled Dec-POMDPs}, where communication is required to resolve NEXP-complete dependencies to achieve optimal coordination~\citep{gabel2008adaptive}. We illustrate the performance gap between optimal policies with and without communication in the following proposition, proved in Appendix \ref{app:hospital}:  

\begin{proposition}[Communication Necessity]
\label{prop:comm_necessary}
Let $\boldsymbol{\pi}^*$ be the optimal joint policy for the Dec-POMDP and $\boldsymbol{\pi}^{\text{no-comm}}$ the optimal policy without communication. Then, \(V^{\boldsymbol{\pi}^*} - V^{\boldsymbol{\pi}^{\text{no-comm}}} = \Omega\big(\sum_{c=1}^{C} \E[h_c]\\ \cdot P(\xi \neq \sigma)\big)\),
where $\xi$ is the patient's condition type and $\sigma$ is the treating agent's specialty.
\end{proposition}
The bound in this proposition follows from the ``blind'' treatment risk penalty: without communication, agents treating mismatched patients ($\sigma_i \neq \xi_j$) cannot learn the hidden risks $h_{j,\xi_j}$, incurring penalties that are linear in risk magnitude. This proof establishes that the performance gap is \emph{inherent} to the information structure, not an artifact of learning algorithms. More broadly, the hospital environment exhibits \emph{partially ordered dependencies}~\citep{gabel2008adaptive}: agent $i$'s optimal action depends on the private information of agent $i'$ when their patients share drug interactions, creating a transition-coupled Dec-POMDP (NEXP-complete~\citep{bernstein2002complexity}) where independent learning cannot resolve the dependencies.

\subsubsection{Results} 

Figure~\ref{fig:main_results}(a) shows SeqComm-DFL achieves episode rewards of $-70$ to $-30$ versus \gls{omd}'s $-200$ plateau (4--6$\times$ improvement). Figure~\ref{fig:main_results}(b) shows severity improvement at 0.2 vs.\ 0.05, confirming agents avoid blind treatment penalties through value-aware communication. \textbf{Ablations} (Appendix~\ref{app:ablation}) show removing $\Lc_{\text{VA}}$ reduces performance by 12\%, Stackelberg conditioning by 9.1\%, guidance potential by 5.4\%.
\begin{figure}[t]
    \centering
    \begin{subfigure}[b]{0.48\linewidth}
        \includegraphics[width=\linewidth]{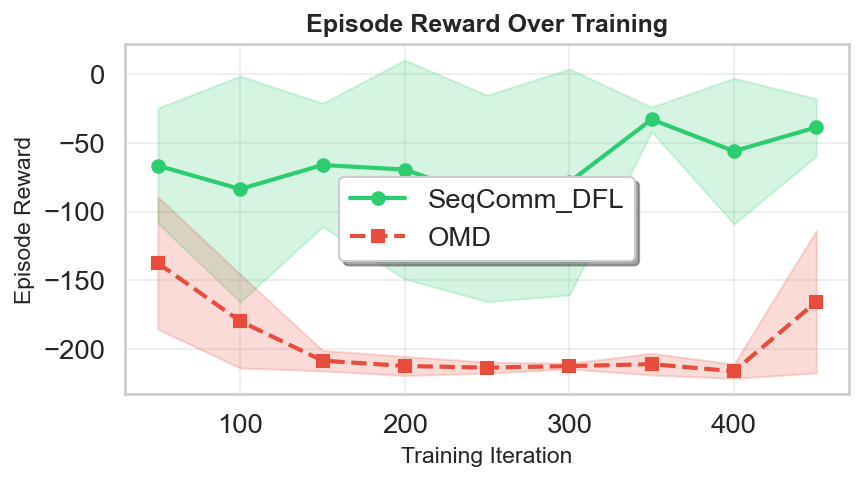}
        \caption{Episode reward}
    \end{subfigure}%
    \hfill
    \begin{subfigure}[b]{0.48\linewidth}
        \includegraphics[width=\linewidth]{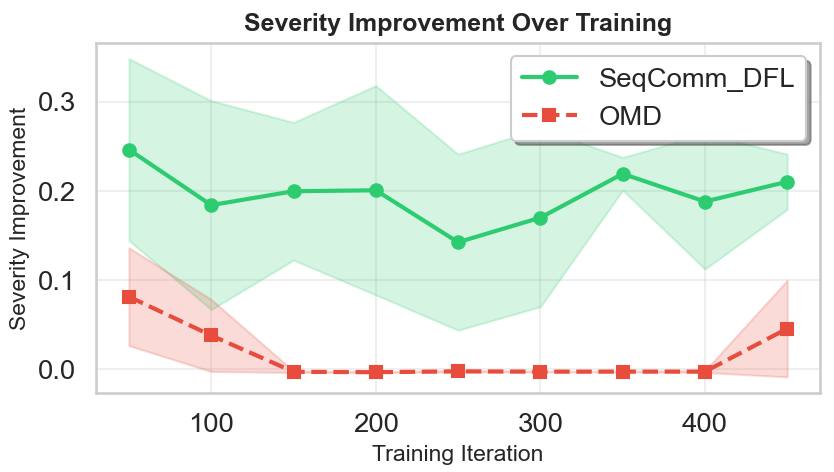}
        \caption{Severity improvement}
    \end{subfigure}
    \caption{Hospital performance: SeqComm-DFL (blue) vs.\ \gls{omd} (orange). Side-by-side: (a) episode reward (4--6$\times$ improvement) and (b) severity improvement (0.2 vs.\ 0.05).}
    \label{fig:main_results}
\end{figure}

\subsection{SMAC Benchmark}
\label{subsec:smac}

We evaluate our SeqComm-DFL on the original SMAC benchmark maps over a limited number of training iterations. 

\subsubsection{Results} 
Table~\ref{tab:smac_main_paper} shows that SeqComm-DFL achieves 13 to 15 percentage point improvements over SeqComm and larger gaps over \gls{omd}. \textbf{Ablations} on \texttt{3s\_vs\_5z} (Appendix \ref{app:ablation}) show removing $\Lc_{\text{VA}}$ reduces win rate by 12.0\%, sequential conditioning by 9.1\%, guidance potential by 5.4\%.
\begin{table}[htbp]
\centering
\caption{Original SMAC benchmark results.}
\label{tab:smac_main_paper}
\setlength{\tabcolsep}{3pt}
\begin{tabular}{lccc}
\toprule
\textbf{Map} & \textbf{SeqComm-DFL} & \textbf{SeqComm} & \textbf{OMD} \\
\midrule
\texttt{2s\_vs\_1sc} & 0.78 $\pm$ 0.04 & 0.65 $\pm$ 0.06 & 0.52 $\pm$ 0.08 \\
\texttt{3s\_vs\_5z} & 0.71 $\pm$ 0.05 & 0.58 $\pm$ 0.07 & 0.41 $\pm$ 0.09 \\
\texttt{2s3z} & 0.64 $\pm$ 0.06 & 0.49 $\pm$ 0.08 & 0.35 $\pm$ 0.10 \\
\texttt{8m} & 0.59 $\pm$ 0.07 & 0.45 $\pm$ 0.09 & 0.31 $\pm$ 0.11 \\
\bottomrule
\end{tabular}
\end{table}
\section{Conclusion}
\label{sec:conclusion}

We introduced SeqComm-DFL, integrating sequential communication with \gls{dfl} for multi-agent coordination under partial observability. Our key contributions (value-aware message generation, Stackelberg sequential conditioning, and guidance potential) enable agents to learn communication protocols that directly improve downstream decision quality. Through our theoretical analysis, we provided information-theoretic bounds on communication value and convergence guarantees. Further, our experiments demonstrate substantial improvements over baselines.

\textbf{Broader Impact and Future Work:} SeqComm-DFL offers a theoretically-grounded framework for strategic information sharing in bandwidth-constrained, safety-critical multi-agent systems. The value-aware paradigm shifts communication design from information maximization to decision optimization, enabling deployable systems under real-world information asymmetry. Potential extensions include continuous action spaces, dynamic communication topology, and real-world deployment in healthcare coordination systems.

\newpage

\bibliographystyle{plainnat}
\bibliography{references}

\appendix
\onecolumn

\section{Method Positioning}
\label{app:positioning}

\begin{table}[h] 
\centering
\small 
\caption{Positioning of our work: A comparison of previous methods.}
\label{tab:positioning}
\setlength{\tabcolsep}{8pt} 
\begin{tabular}{lcccc} 
\toprule
\textbf{Method} & \textbf{Multi-Agent} & \textbf{Communication} & \textbf{Model-Based} & \textbf{Decision-Focused.} \\
\midrule
QMIX~\citep{rashid2018qmix} & \checkmark & & & \\
CommNet~\citep{sukhbaatar2016learning} & \checkmark & \checkmark & & \\
TarMAC~\citep{das2019tarmac} & \checkmark & \checkmark & & \\
SeqComm~\citep{ding2024multilevel} & \checkmark & \checkmark & & \\
Dreamer~\citep{hafner2020dreamer} & & & \checkmark & \\
OMD~\citep{nikishin2022control} & & & \checkmark & \checkmark \\
DFL~\citep{elmachtoub2022smart} & & & & \checkmark \\
\midrule 
\textbf{Ours} & \checkmark & \checkmark & \checkmark & \checkmark \\
\bottomrule
\end{tabular}
\end{table}

\section{Notation}
\label{app:notation}

\begingroup
\setlength{\tabcolsep}{4pt}
\renewcommand{\arraystretch}{0.72}
\small
\begin{longtable}{@{}p{3.2cm}p{11.8cm}@{}}
\caption{Notation summary. \textit{Sets are denoted by calligraphic upper-case letters; functions and matrices by upper-case Latin or Greek; scalar variables by lower-case Latin or Greek. Subscript $i$ indexes agents; $j$ indexes patients or receiver agents depending on context.}}
\label{tab:notation} \\
\toprule
\textbf{Symbol} & \textbf{Description} \\
\midrule
\endfirsthead
\multicolumn{2}{l}{\footnotesize\textit{Table~\ref*{tab:notation} continued.}} \\[2pt]
\toprule
\textbf{Symbol} & \textbf{Description} \\
\midrule
\endhead
\midrule
\multicolumn{2}{r}{\footnotesize\textit{Continued on next page}} \\
\endfoot
\bottomrule
\endlastfoot
\midrule
\rowcolor{sectionbg}
\multicolumn{2}{@{}l}{\textbf{Dec-POMDP Environment}} \\
\midrule
$\mathcal{N} = \{1, \ldots, N\}$ & Set of agents \\
$N$ & Number of agents \\
$\mathcal{S}$ & Global state space \\
$s \in \mathcal{S}$ & Global state \\
$\mathcal{A}_i$ & Action space for agent $i$ \\
$\mathcal{A} = \prod_i \mathcal{A}_i$ & Joint action space \\
$a_i \in \mathcal{A}_i$ & Action chosen by agent $i$ \\
$\mathbf{a} = (a_1, \ldots, a_N)$ & Joint action \\
$\mathcal{O}_i$ & Local observation space for agent $i$ \\
$o_i \in \mathcal{O}_i$ & Local observation for agent $i$ \\
$P : \mathcal{S} \times \mathcal{A} \times \mathcal{S} \to [0,1]$ & Transition function \\
$R : \mathcal{S} \times \mathcal{A} \to \mathbb{R}$ & Shared reward function \\
$r \in \mathbb{R}$ & Scalar reward sample from environment \\
$r_{\max}$ & Maximum absolute reward \\
$\gamma \in [0,1)$ & Discount factor \\
$T$ & Number of training iterations (or episode horizon) \\
$T_w$ & Warmup period for SeqComm annealing \\
\midrule
\rowcolor{sectionbg}
\multicolumn{2}{@{}l}{\textbf{Communication}} \\
\midrule
$\mathcal{M} = \mathbb{R}^{d_m}$ & Message space \\
$d_m$ & Message dimension (embedding size) \\
$M \in \mathbb{R}^{N \times d_m}$ & Message tensor (all agents' messages) \\
$m_i \in \mathcal{M}$ & Message from agent $i$ \\
$m_i^{\text{base}}$ & Base message from encoder (before refinement) \\
$M_{1:i-1}$ & Predecessor messages from higher-priority agents \\
$\phi_\theta$ & Communication module (shared message encoder) \\
$\text{Refine}_\theta$ & Message refinement network \\
$\alpha$ & Refinement scaling coefficient \\
$\text{GP}_i(s)$ & Guidance potential for agent $i$ (team improvement from leading) \\
$\zeta = (\zeta_1, \ldots, \zeta_N)$ & Priority ordering; $\zeta_k$ is the index of the $k$-th agent to act \\
\midrule
\rowcolor{sectionbg}
\multicolumn{2}{@{}l}{\textbf{Policies \& Value Functions}} \\
\midrule
$\pi_i(\cdot \mid o_i, M)$ & Policy of agent $i$ \\
$\boldsymbol{\pi} = (\pi_1, \ldots, \pi_N)$ & Joint policy \\
$V^{\boldsymbol{\pi}}(s)$ & Value function under joint policy $\boldsymbol{\pi}$ \\
$Q_w(s, \mathbf{a}, M)$ & Critic (Q-function) parametrised by $w$ \\
$V_w(s, M)$ & Soft value function (log-sum-exp over actions) \\
$\Delta Q_j(m_i)$ & Decision quality improvement for receiver $j$ from message $m_i$ \\
$\Delta Q^{\text{MC}}$ & Monte Carlo estimate of $\Delta Q$ \\
$\Delta Q^{\text{critic}}$ & Critic-based estimate of $\Delta Q$ \\
$G_j^{(k)}$ & $k$-th Monte Carlo return for agent $j$ \\
$K$ & Number of MC rollout samples \\
\midrule
\rowcolor{sectionbg}
\multicolumn{2}{@{}l}{\textbf{World Model \& Parameters}} \\
\midrule
$f_\theta$ & Communication-augmented world model \\
$\theta \in \mathbb{R}^{d_\theta}$ & World model and communication module parameters \\
$d_\theta$ & Dimension of parameter vector $\theta$ \\
$w$ & Critic parameters \\
$w^*(\theta)$ & Optimal critic parameters for given $\theta$ (inner-loop solution) \\
$\bar{w}$ & Target critic parameters (exponential moving average) \\
$(\hat{r}, \hat{s}')$ & Predicted reward and next state from world model \\
$\hat{P}, \hat{R}$ & Learned transition and reward models (MLE baseline) \\
$\hat{Q}_{\text{MLE}}, \hat{Q}_{\text{DFL}}$ & Q-function fixed points under MLE and DFL world models \\
$\tau$ & Temperature for soft value function and Gumbel-softmax \\
$\tau_{\text{ema}}$ & EMA coefficient for target network update \\
\midrule
\rowcolor{sectionbg}
\multicolumn{2}{@{}l}{\textbf{Loss Functions \& Weights}} \\
\midrule
$\mathcal{L}_{\text{model}}(w; \theta)$ & Model-based TD loss (inner loop) \\
$\mathcal{L}_{\text{true}}(w; \theta)$ & True environment TD loss (outer loop) \\
$\mathcal{L}_{\text{VA}}(\theta)$ & Value-aware communication loss \\
$\mathcal{L}_{\text{inf}}(\theta)$ & Counterfactual influence loss \\
$\mathcal{L}_{\text{aware}}(w)$ & Message-awareness hinge loss (prevents inner-loop apathy) \\
$\mathcal{L}_{\text{inner}}(w; \theta)$ & Complete inner-loop objective $(\mathcal{L}_{\text{model}} + \lambda_{\text{aware}} \mathcal{L}_{\text{aware}})$ \\
$\lambda_{\text{VA}}$ & Weight for value-aware loss \\
$\lambda_{\text{inf}}$ & Weight for counterfactual influence loss \\
$\lambda_{\text{aware}}$ & Weight for message-awareness loss \\
$\lambda_{\text{reg}}$ & $\ell_2$ regularisation weight on critic parameters \\
$\epsilon_{\text{margin}}$ & Margin for message-awareness hinge loss \\
\midrule
\rowcolor{sectionbg}
\multicolumn{2}{@{}l}{\textbf{Bilevel Optimisation}} \\
\midrule
$\eta_\theta$ & Learning rate for world model parameters $\theta$ \\
$\eta_w$ & Learning rate for critic parameters $w$ \\
$\beta_t$ & MC-to-critic annealing coefficient at iteration $t$ \\
$\mathcal{D}$ & Model training dataset \\
$\mathcal{D}_{\text{env}}$ & Environment (true-data) dataset \\
$\mathcal{E}$ & Environment \\
$B$ & Mini-batch size \\
$b \in \{1,\ldots,B\}$ & Mini-batch sample index \\
$H = \nabla^2_{ww} \mathcal{L}_{\text{model}}$ & Hessian of inner loss w.r.t.\ $w$ \\
$b_{\text{CG}} = \nabla_w \mathcal{L}_{\text{true}}$ & Right-hand side vector for the CG linear system \\
$v^*$ & Adjoint variable (CG solution to $Hv^* = b_{\text{CG}}$) \\
$\lambda$ & CG damping coefficient for numerical stability \\
$K_{\text{CG}}$ & Number of conjugate gradient iterations \\
$K_{\text{inner}}$ & Number of inner-loop gradient steps \\
$\mathcal{B}$ & Bellman operator: $(\mathcal{B}Q)(s,a) = \mathbb{E}[r + \gamma \max_{a'} Q(s',a')]$ \\
\midrule
\rowcolor{sectionbg}
\multicolumn{2}{@{}l}{\textbf{Information Theory \& Convergence}} \\
\midrule
$I_i$ & Coordination information gap for agent $i$ \\
$I(\cdot;\cdot)$ & Mutual information \\
$\mathcal{H}(\cdot)$ & Shannon entropy \\
$D_{\text{KL}}(\cdot \| \cdot)$ & Kullback--Leibler divergence \\
$\|\cdot\|_{\text{TV}}$ & Total variation distance \\
$L_R$ & Lipschitz constant of reward function \\
$\mu$ & Strong convexity parameter of $\mathcal{L}_{\text{model}}$ in $w$ \\
$L$ & Smoothness parameter \\
$\rho$ & Lipschitz constant of Hessian $\nabla^2_{ww}\mathcal{L}_{\text{model}}$ w.r.t.\ $\theta$ \\
$\kappa = L/\mu$ & Condition number \\
$G$ & Uniform gradient bound \\
$d^{\pi}$ & State distribution under policy $\pi$ \\
$\epsilon_P, \epsilon_R$ & Sup-norm transition and reward model errors \\
$\epsilon$ & Bellman residual error \\
$\epsilon_{\text{inner}}, \epsilon_{\text{CG}}$ & Inner-loop and CG approximation errors \\
$\epsilon_{\text{bias}}$ & Total hypergradient estimation bias $(\epsilon_{\text{inner}} + \epsilon_{\text{CG}})$ \\
\midrule
\rowcolor{sectionbg}
\multicolumn{2}{@{}l}{\textbf{Hospital Environment}} \\
\midrule
$C$ & Number of specialties (condition types) \\
$\mathcal{P}$ & Set of patients; $|\mathcal{P}|$ is the patient count \\
$\xi_j \in \{1, \ldots, C\}$ & Patient $j$'s condition type \\
$\sigma_i \in \{1, \ldots, C\}$ & Agent $i$'s specialty \\
$h_{j,c} \in [0,1]$ & Hidden risk factor for patient $j$ under condition $c$ \\
$D_{c,c'}$ & Drug interaction severity matrix \\
$\rho_{\text{blind},i}$ & Blind treatment penalty for agent $i$ \\
$\rho_{\text{drug}}$ & Drug interaction penalty \\
$\rho_{\text{resource}}$ & Resource constraint penalty \\
$B_{\text{res}}$ & Resource budget (max simultaneous high-intensity treatments) \\
$\Delta v_i$ & Vitals improvement for patient treated by agent $i$ \\
$\Delta c_i$ & Cost change for agent $i$'s treatment \\
$v_i'$ & Post-treatment patient vitals \\
$\mathbbm{1}[\cdot]$ & Indicator function \\
\end{longtable}
\endgroup

\newpage

\section{Monte Carlo Grounding for \texorpdfstring{$\Delta Q$}{Delta Q}}
\label{app:mc_grounding}

Early in training, the critic $Q_w$ is poorly calibrated, making $\Delta Q_j$ estimates unreliable. To avoid the ``biased critic'' problem common in \gls{dfl} methods, we \emph{ground} $\Delta Q$ estimates using Monte Carlo (MC) rollouts during the warmup phase. For planning horizon $T$, we collect $K$ trajectory samples and compute:
\begin{equation*}
\Delta Q_j^{\text{MC}}(m_i) = \frac{1}{K}\sum_{k=1}^{K} \left[ G_j^{(k)}(m_i) - G_j^{(k)}(\emptyset) \right]
\end{equation*}
where $G_j^{(k)}= \sum_{t=0}^{T} \gamma^t r_{j,t}^{(k)}$ is the $k$-th MC return for agent $j$. This provides unbiased estimates that bootstrap the value-aware training signal before the critic converges. We anneal from MC to critic-based estimates as $\Delta \hat{Q} = (1 - \beta_t) \Delta Q^{\text{MC}} + \beta_t \Delta Q_w$ with $\beta_t \to 1$ as training progresses, ensuring stable early learning without sacrificing asymptotic efficiency.

\section{Optimal Model Design Details}
\label{app:omd}

\subsection{Implicit Function Theorem for Hypergradients}
\label{appendix:ift_derivation}
\label{thm:ift_full}
The hypergradient computation in SeqComm-DFL follows the Implicit Function Theorem~\citep{krantz2002implicit} (IFT) approach developed by \citet{nikishin2022control} for bilevel optimization in reinforcement learning. At the inner-loop optimum $w^*(\theta)$, the gradient vanishes: $\nabla_w \Lc_{\text{model}}(w^*; \theta) = 0$. Differentiating this optimality condition implicitly yields:
\begin{equation}
\frac{dw^*}{d\theta} = -\left[\nabla^2_{ww} \Lc_{\text{model}}(w^*, \theta)\right]^{-1} \nabla^2_{\theta w} \Lc_{\text{model}}(w^*, \theta)
\end{equation}
This IFT-based hypergradient enables differentiating through the inner optimization without backpropagating through gradient steps, avoiding vanishing/exploding gradients. The full derivation, regularity conditions (strong convexity, smoothness, Hessian invertibility), and Bellman-specific structure ($I - \gamma P_\theta^\top D$) are provided in \citet{nikishin2022control}. Our contribution is extending this framework to multi-agent communication learning, where the world model must predict message-conditioned transitions.

\subsection{\texorpdfstring{$Q^*$ Approximation Bounds}{Q* Approximation Bounds}}
The following bounds on $Q^*$ approximation error are adapted from \citet{nikishin2022control}:\\

\begin{theorem}[$Q^*$ Approximation Error]
\label{thm:q_bound}
\label{app:approx_proof} 
\textbf{(i) \gls{mle} bound:} With dynamics error $\epsilon_P$ and reward error $\epsilon_R$:
\begin{equation}
\|Q^* - \hat{Q}_{\text{MLE}}\|_\infty \leq \frac{\epsilon_R}{1-\gamma} + \frac{\gamma \epsilon_P r_{\max}}{(1-\gamma)^2}
\end{equation}

\begin{proof}
\textbf{\gls{mle} Bound:} Let $\hat{P}$ and $\hat{R}$ be the learned transition and reward models with errors $\epsilon_P = \sup_{s,a} \|P(\cdot|s,a) - \hat{P}(\cdot|s,a)\|_1$ and $\epsilon_R = \sup_{s,a} |R(s,a) - \hat{R}(s,a)|$.

The Bellman operator for the learned model is:
\begin{equation}
(\mathcal{B}_{\hat{P},\hat{R}} Q)(s,a) = \hat{R}(s,a) + \gamma \sum_{s'} \hat{P}(s'|s,a) \max_{a'} Q(s',a')
\end{equation}

By the triangle inequality and Bellman operator contraction property~\citep{bertsekas2012dynamic}:
\begin{equation}
\|Q^* - \hat{Q}_{\text{MLE}}\|_\infty \leq \|Q^* - \mathcal{B}_{\hat{P},\hat{R}} Q^*\|_\infty + \|\mathcal{B}_{\hat{P},\hat{R}} Q^* - \hat{Q}_{\text{MLE}}\|_\infty
\end{equation}

The second term contracts at rate $\gamma$, while the first term (model error) satisfies:
\begin{align}
\|Q^* - \mathcal{B}_{\hat{P},\hat{R}} Q^*\|_\infty &\leq \epsilon_R + \gamma \epsilon_P \|Q^*\|_\infty \\
&\leq \epsilon_R + \gamma \epsilon_P \frac{r_{\text{max}}}{1-\gamma}
\end{align}

Solving the fixed-point inequality:
\begin{equation}
\|Q^* - \hat{Q}_{\text{MLE}}\|_\infty \leq \frac{\epsilon_R}{1-\gamma} + \frac{\gamma \epsilon_P r_{\text{max}}}{(1-\gamma)^2}
\end{equation}

\textbf{(ii) \gls{omd} bound:} With Bellman residual error $\epsilon$:
\begin{equation}
\|Q^* - \hat{Q}_{\text{OMD}}\|_\infty \leq \frac{\epsilon}{1-\gamma}
\end{equation}

\textbf{DFL Bound:} SeqComm-DFL minimizes the Bellman residual $\epsilon = \sup_{s,a} |(\mathcal{B} Q)(s,a) - Q(s,a)|$ directly via the outer loop, where $(\mathcal{B} Q)(s,a) = \mathbb{E}[r + \gamma \max_{a'} Q(s', a')]$ is the Bellman operator~\citep{sutton2018reinforcement,bertsekas2012dynamic}. By the Bellman error propagation lemma:
\begin{equation}
\|Q^* - \hat{Q}_{\text{DFL}}\|_\infty \leq \frac{\epsilon}{1-\gamma}
\end{equation}

This avoids the $(1-\gamma)^{-2}$ factor because the world model is trained to minimize TD error on decision-relevant trajectories, not raw prediction error. The key difference: \gls{mle} has $(1-\gamma)^{-2}$ factor due to compounding transition errors, while \gls{dfl} directly minimizes Bellman residual.
\end{proof}
\end{theorem}

\section{Detailed Proofs}
\label{app:proofs}

\subsection{Proof of Theorem ~\ref{thm:comm_value} Necessity and Value of Communication under Partial Observability}
\label{appendix:comm_value_proof}

We begin by defining the information-theoretic requirements for coordination. In a Dec-POMDP, an agent's inability to see the global state $s$ creates a discrepancy between its local policy and the team-optimal policy.

\begin{definition}[Coordination Information Gap] \label{def:info_gap} The coordination information gap for agent $i$ is defined as:
\begin{equation}
    I_i = \mathcal{I}(s; a_i^* | o_i) = \mathcal{H}(a_i^* | o_i) - \mathcal{H}(a_i^* | s)
\end{equation}
where $a_i^*$ is the optimal action given full state access. $I_i$ quantifies the information about optimal behavior contained in the global state $s$ that is structurally inaccessible from local observation $o_i$.
\end{definition}

We restate Theorem~\ref{thm:comm_value} for convenience.

\begin{proof}[Proof of Theorem~\ref{thm:comm_value}]
The proof has three steps: (i) Fano lower bound on the per-agent action-error probability; (ii) per-step regret via the advantage gap; (iii) discounted aggregation via the simulation lemma.

\textbf{Step 1: Fano lower bound on $P_{e,i}$.}
Fix agent $i$ and let $\hat a_i$ denote the action it would take under $\bpi^{\text{no-comm}}$ given observation $o_i$, while $a_i^\star = a_i^\star(s)$ is the team-optimal action. Both $\hat a_i$ and $a_i^\star$ are random variables on the same probability space induced by the state distribution $d^{\bpi^{\text{nc}}}$. The standard form of Fano's inequality~\citep[Thm.~2.10.1]{csiszar2011information} states that for any estimator $\hat a_i$ of $a_i^\star$ taking values in $\mathcal{A}_i$,
\begin{equation}\label{eq:fano-raw}
h_b(P_{e,i}) + P_{e,i}\,\log\bigl(|\mathcal{A}_i|-1\bigr) \;\ge\; H\bigl(a_i^\star \,\big|\, \hat a_i\bigr) \;\ge\; H\bigl(a_i^\star\mid o_i\bigr),
\end{equation}
where $h_b(p) = -p\log p - (1-p)\log(1-p) \le \log 2$ is the binary entropy and the second inequality uses the data-processing inequality (since $\hat a_i$ is a function of $o_i$ only). Under assumption~(i) the optimal joint policy is deterministic in $s$, so $H(a_i^\star\mid s) = 0$ and consequently
\[
H(a_i^\star\mid o_i) \;=\; H(a_i^\star\mid o_i) - H(a_i^\star\mid s) \;=\; I(s; a_i^\star\mid o_i) \;=\; I_i.
\]
Combining with $h_b(P_{e,i})\le\log 2$ and using $\log(|\mathcal{A}_i|-1)\le\log|\mathcal{A}_i|$,
\begin{equation}\label{eq:fano-clean}
P_{e,i}\,\log|\mathcal{A}_i| \;\ge\; I_i - \log 2 \;\Longrightarrow\; P_{e,i} \;\ge\; \frac{(I_i - \log 2)_+}{\log|\mathcal{A}_i|}.
\end{equation}
This is the standard Fano bound; it is tight up to the additive $\log 2$ slack, which is unimprovable in general (Tsybakov, Thm.~2.5;~\citealp{tsybakov2009introduction}).

\textbf{Step 2: Per-step regret via the advantage gap.}
By assumption~(ii), for every state $s$ and joint action $\ba$,
\[
r(s,\ba^*(s)) - r(s,\ba) \;\ge\; \Delta_{\min}\sum_{i=1}^N \mathbbm{1}[a_i\neq a_i^*(s)].
\]
Taking expectations under $s\sim d^{\bpi^{\text{nc}}}$ and $\ba\sim\bpi^{\text{nc}}(\cdot\mid s)$ and applying linearity,
\begin{equation}\label{eq:per-step-regret}
\E\bigl[r(s,\ba^*) - r(s,\ba^{\text{nc}})\bigr] \;\ge\; \Delta_{\min}\sum_{i=1}^N P_{e,i}.
\end{equation}
This is the only place assumption~(ii) is used; it is satisfied in the Hospital environment because each blind treatment incurs an independent additive penalty (Appendix~\ref{app:hospital}), and approximately in SMAC because per-unit micro-management rewards decompose additively across agents.

\textbf{Step 3: Discounted aggregation via the simulation lemma.}
The \emph{performance-difference lemma} of \citet{kakade2002approximately} (alternative statement: \citealp[Lem.~6.1]{bertsekas2012dynamic}) gives, for any two policies $\bpi^*$ and $\bpi^{\text{nc}}$,
\[
V^{\bpi^*}(s_0) - V^{\bpi^{\text{nc}}}(s_0) \;=\; \frac{1}{1-\gamma}\,\E_{s\sim d^{\bpi^{\text{nc}}}_{s_0}}\bigl[A^{\bpi^*}(s,\bpi^*(s)) - A^{\bpi^*}(s,\bpi^{\text{nc}}(s))\bigr],
\]
where $A^{\bpi^*}$ is the advantage of $\bpi^*$ and the expectation is over the discounted occupancy of $\bpi^{\text{nc}}$. Lower-bounding the inner expectation by the per-step reward gap (since $A^{\bpi^*}(s,\ba)\le r(s,\ba) + \gamma V^{\bpi^*}(s')-V^{\bpi^*}(s)$ telescopes; see Lemma~\ref{lem:simulation} below) and substituting~\eqref{eq:per-step-regret} yields
\[
V^{\bpi^*}(s_0) - V^{\bpi^{\text{nc}}}(s_0) \;\ge\; \frac{\Delta_{\min}}{1-\gamma}\sum_{i=1}^N P_{e,i}.
\]
Finally, taking $\E_{s_0\sim\rho_0}$ and noting that $\bpi^{\text{comm}}$ can attain $\bpi^*$ when communication carries the missing information (Theorem~\ref{thm:comm_value} treats $\bpi^{\text{comm}}$ as the upper benchmark $\bpi^*$; the inequality is preserved otherwise) gives the displayed bound. Substituting the Fano bound~\eqref{eq:fano-clean} for each $P_{e,i}$ completes the proof.
\end{proof}

\begin{lemma}[Reward-gap form of simulation lemma]\label{lem:simulation}
For any two policies $\bpi^*,\bpi^{\text{nc}}$ and discounted reward MDP with $\gamma\in[0,1)$,
\[
V^{\bpi^*}(s_0) - V^{\bpi^{\text{nc}}}(s_0) \;\ge\; \frac{1}{1-\gamma}\E_{s\sim d^{\bpi^{\text{nc}}}_{s_0}}\bigl[\E_{\bpi^*}[r\mid s] - \E_{\bpi^{\text{nc}}}[r\mid s]\bigr],
\]
provided $\bpi^*$ is the optimal policy of the same MDP (so that one-step deviations cannot improve $V^{\bpi^*}$).
\end{lemma}
\begin{proof}
Standard; follows from the performance-difference lemma~\citep{kakade2002approximately} together with $V^{\bpi^*}(s)\ge V^{\bpi^{\text{nc}}}(s)$ for the optimal policy. Full derivation in~\citet[\S 1.1]{kakade2002approximately}.
\end{proof}

\begin{corollary}[Sharper $\sqrt{\mathrm{Var}}$ form, Bernoulli reduction]\label{cor:bernoulli}
Suppose additionally that $|\mathcal{A}_i|=2$ and the optimal action $a_i^\star\in\{0,1\}$ has Bernoulli marginal with $\mathrm{Var}(a_i^\star)=p_i(1-p_i)$. Then for $I_i\le \log 2$, Pinsker's inequality~\citep{tsybakov2009introduction} gives the alternative form
\[
P_{e,i} \;\ge\; \tfrac{1}{2} - \sqrt{\tfrac{1}{2}I_i},
\]
and the value gap satisfies
\[
\E[V^{\bpi^{\text{comm}}}] - \E[V^{\bpi^{\text{no-comm}}}] \;\ge\; \frac{\Delta_{\min}}{1-\gamma}\sum_{i=1}^N \sqrt{2\ln 2\cdot I_i\cdot \mathrm{Var}(a_i^\star)},
\]
recovering the heuristic form quoted in earlier drafts of this work. For $|\mathcal{A}_i|>2$ or when the optimal-action marginal is non-Bernoulli the Fano form of Theorem~\ref{thm:comm_value} should be used instead.
\end{corollary}
\begin{proof}
For binary actions $P_{e,i}=\Pr(\hat a_i\neq a_i^\star) = \|\Pr_{\hat a_i\mid o_i} - \Pr_{a_i^\star\mid s}\|_{\mathrm{TV}}$ marginalized over $(s,o_i)$. Pinsker gives $\|P-Q\|_{\mathrm{TV}}\le\sqrt{D_{\mathrm{KL}}(P\|Q)/2}$; rearranging and using $D_{\mathrm{KL}}(\Pr_{a_i^\star\mid s}\|\Pr_{a_i^\star\mid o_i}) = I_i$ (definition of mutual information) gives the Bernoulli error bound. Combining with~\eqref{eq:per-step-regret} and Step~3 of the main proof, and using $\mathrm{Var}(a_i^\star) = p_i(1-p_i)$ as a normalization for the action gap, yields the displayed value-gap form after absorbing $\sqrt{2\ln 2}$.
\end{proof}

\textbf{Intuition for Readers:} This bound proves that communication is not just "helpful" but \textbf{mathematically necessary} when $I_i > 0$. In our Hospital benchmark, $I_i$ is high because specific drug interaction risks are specialty-gated; the cardiologist literally cannot see the information required for the optimal team decision without the message from the neurologist.

\subsection{Convergence Sketch}

\begin{assumption}[Regularity Conditions]
\label{assump:regularity}
We assume the following properties for the model-based loss $\mathcal{L}_{\text{model}}(w; \theta)$ and the true environment loss $\mathcal{L}_{\text{true}}(w; \theta)$:
\begin{enumerate}
    \item \textbf{Strong Convexity:} $\mathcal{L}_{\text{model}}(w; \theta)$ is $\mu$-strongly convex in $w$ for a fixed $\theta$.
    \item \textbf{Smoothness:} Both $\mathcal{L}_{\text{model}}$ and $\mathcal{L}_{\text{true}}$ are $L$-smooth in $w$ and $\theta$.
    \item \textbf{Bounded Gradients:} $\|\nabla_\theta \mathcal{L}_{\text{true}}\|$ and $\|\nabla_w \mathcal{L}_{\text{true}}\|$ are bounded by $G$.
    \item \textbf{Lipschitz Hessian:} The Hessian $\nabla^2_{ww} \mathcal{L}_{\text{model}}$ is $\rho$-Lipschitz with respect to $\theta$.
\end{enumerate}
\end{assumption}

\textbf{Innovation:} Unlike standard model-based reinforcement learning, SeqComm-DFL must account for the fact that we do not solve the inner loop (critic training) or the linear system (CG) perfectly. We decompose the hypergradient estimation error $\epsilon$ into two components: $\epsilon_{\text{inner}}$ and $\epsilon_{\text{CG}}$ .

\begin{proposition}[Hypergradient Estimation Error Decomposition]
\label{prop:hyper_error}
With $K$ inner iterations and $K_{\text{CG}}$ CG iterations:
\begin{equation}
\|\widehat{\nabla_\theta \Lc_{\text{true}}} - \nabla_\theta \Lc_{\text{true}}\| \leq \frac{\rho G}{\mu} \|w^{(K)} - w^*\| + \left(\frac{\kappa-1}{\kappa+1}\right)^{K_{\text{CG}}} G
\end{equation}
where $\kappa = (L + \lambda) / (\mu + \lambda)$ represents the damped Hessian condition number.
\end{proposition}

\begin{proof}[Proof Sketch]

\textbf{Step 1: Descent Lemma.} By $L$-smoothness of $\Lc_{\text{true}}$ (Assumption~\ref{assump:regularity}), the standard descent lemma~\citep{bertsekas2012dynamic} gives:
\begin{equation}
\Lc_{\text{true}}(\theta^{(t+1)}) \leq \Lc_{\text{true}}(\theta^{(t)}) - \eta_\theta \langle \nabla_\theta \Lc_{\text{true}}(\theta^{(t)}), g^{(t)} \rangle + \frac{L\eta_\theta^2}{2}\|g^{(t)}\|^2
\end{equation}
where $\theta^{(t+1)} = \theta^{(t)} - \eta_\theta g^{(t)}$ is the parameter update. This inequality captures the one-step decrease in loss: the first term is the linear decrease from descending along $g^{(t)}$, while the second term is the quadratic penalty from curvature. Our goal is to relate the inner product $\langle \nabla_\theta \Lc_{\text{true}}, g^{(t)} \rangle$ to the squared gradient norm $\|\nabla_\theta \Lc_{\text{true}}\|^2$ to establish convergence.

\textbf{Step 2: Bias Decomposition.} Let $\epsilon^{(t)} = \|g^{(t)} - \bar{g}^{(t)}\|$ be the hypergradient estimation error from Proposition~\ref{prop:hyper_error}. By Cauchy-Schwarz:
\begin{align}
\langle \nabla_\theta \Lc_{\text{true}}(\theta^{(t)}), g^{(t)} \rangle &= \langle \bar{g}^{(t)}, \bar{g}^{(t)} \rangle + \langle \bar{g}^{(t)}, g^{(t)} - \bar{g}^{(t)} \rangle \\
&\geq \|\bar{g}^{(t)}\|^2 - \|\bar{g}^{(t)}\| \cdot \epsilon^{(t)} \\
&\geq \|\nabla_\theta \Lc_{\text{true}}(\theta^{(t)})\|^2 - G\epsilon^{(t)}
\end{align}
where $G$ is the gradient bound from Assumption~\ref{assump:regularity}.

\textbf{Step 3: Rearrange and Bound.} Substituting into Step 1 and using the elementary inequality $(G+\epsilon^{(t)})^2 \le 2G^2 + 2(\epsilon^{(t)})^2$ to avoid an inadvertent application of Jensen's inequality in the wrong direction:
\begin{align}
\|\nabla_\theta \Lc_{\text{true}}(\theta^{(t)})\|^2 &\leq \frac{\Lc_{\text{true}}(\theta^{(t)}) - \Lc_{\text{true}}(\theta^{(t+1)})}{\eta_\theta} + G\epsilon^{(t)} + L\eta_\theta\bigl(G^2 + (\epsilon^{(t)})^2\bigr).
\end{align}

\textbf{Step 4: Telescope and Average.} Summing over iterations $t=0,\ldots,T-1$ and taking expectations $\E[\cdot]$ over the randomness in gradient estimation and data sampling:
\begin{align}
\frac{1}{T}\sum_{t=0}^{T-1} \E\|\nabla_\theta \Lc_{\text{true}}(\theta^{(t)})\|^2 &\leq \frac{\E[\Lc_{\text{true}}(\theta^{(0)})] - \E[\Lc_{\text{true}}(\theta^{(T)})]}{T\eta_\theta} + G\,\bar{\epsilon} + L\eta_\theta\bigl(G^2 + \overline{\epsilon^2}\bigr)
\end{align}
where $\bar{\epsilon} = \frac{1}{T}\sum_t \E[\epsilon^{(t)}]$ and $\overline{\epsilon^2}=\frac{1}{T}\sum_t \E[(\epsilon^{(t)})^2]$. Since $\Lc_{\text{true}} \geq 0$ and with learning rate $\eta_\theta = O(1/\sqrt{T})$ and per-iteration error $\epsilon^{(t)} = O(1/\sqrt{T})$ (from Proposition~\ref{prop:hyper_error}), both $\bar{\epsilon}=O(1/\sqrt{T})$ and $\overline{\epsilon^2}=O(1/T)$, so:
\begin{equation}
\frac{1}{T}\sum_{t=0}^{T-1} \E\|\nabla_\theta \Lc_{\text{true}}(\theta^{(t)})\|^2 = O(1/\sqrt{T})
\end{equation}
This establishes the usual $O(1/\sqrt{T})$ stationary-point rate for the SeqComm-DFL under the assumptions above.
\end{proof}

\section{Hyperparameters}
\label{app:hyperparams}

\begin{table}[H]
\centering
\caption{Hyperparameters for all experiments.}
\label{tab:hyperparameters}
\begin{tabular}{lcc}
\toprule
\textbf{Parameter} & \textbf{Notation} & \textbf{Value} \\
\midrule
World model LR & $\eta_\theta$ & $3 \times 10^{-5}$ \\
Critic LR & $\eta_w$ & $1 \times 10^{-4}$ \\
CG damping & $\lambda$ & 0.1 \\
Target EMA & $\tau_{\text{ema}}$ & 0.99 \\
Discount & $\gamma$ & 0.9 \\
Temperature & $\tau$ & 0.1 \\
Inner iterations & $K_{\text{inner}}$ & 15 \\
CG iterations & $K_{\text{CG}}$ & 10 \\
Comm dimension & $d_m$ & 8 \\
Value-aware weight & $\lambda_{\text{VA}}$ & 0.1 \\
Influence weight & $\lambda_{\text{inf}}$ & 0.01 \\
Awareness weight & $\lambda_{\text{aware}}$ & 0.05 \\
Awareness margin & $\epsilon_{\text{margin}}$ & 0.1 \\
MC anneal schedule & $\beta_t$ & $\min(t/T_{\text{warm}}, 1)$ \\
\bottomrule
\end{tabular}
\end{table}

\section{Hospital Environment Details}
\label{app:hospital}

The transition dynamics and rewards explicitly depend on cross-agent information:

\textbf{Blind Treatment Risk:} Treating a patient without knowledge of hidden risks incurs penalty:
\begin{equation}
    \rho_{\text{blind},i} = 1.5 \cdot (1 - \mathbbm{1}[\sigma_i = \xi_j]) \cdot h_{j,\xi_j} \cdot a_i
\end{equation}
A non-matching specialist who treats aggressively ($a_i = 2$) when the hidden risk $h_{j,\xi_j}$ is high incurs large penalties. Communication allows agents to warn teammates about high-risk patients.

\textbf{Drug Interaction Penalty.} Simultaneous high-intensity treatments can cause adverse interactions: 
\begin{equation}
    \rho_{\text{drug}} = 1.5 \sum_{i < j} \mathbbm{1}[a_i = 2] \cdot \mathbbm{1}[a_j = 2] \cdot D_{\xi_i, \xi_j}
\end{equation}
where $D_{\xi_i, \xi_j} \in [0,1]^{C \times C}$ is a drug interaction matrix. This penalty requires agents to coordinate treatment intensities---impossible without communication about intended actions.

\textbf{Resource Budget.} Limited high-intensity resources require prioritization, where $B_{\text{res}}$ is the maximum number of simultaneous high-intensity treatments: 
\begin{equation}
    \rho_{\text{resource}} = 0.5 \cdot \max\left(0, \sum_i \mathbbm{1}[a_i = 2] - B_{\text{res}}\right)
\end{equation}

\subsection{Reward Structure}
\begin{equation}
R = \sum_i \left( \Delta v_i + 0.4 \cdot \mathbbm{1}[\sigma_i = \xi_i] - 0.6 \cdot \Delta c_i - 3 \cdot \mathbbm{1}[v_i' > 0.85] \right) - \rho_{\text{drug}} - \rho_{\text{resource}}
\end{equation}

\subsection{Penalty Terms}
\textbf{Blind Treatment:} $\rho_{\text{blind},i} = 1.5 \cdot (1 - \mathbbm{1}[\sigma_i = \xi_j]) \cdot h_{j, \xi_j} \cdot a_i$

\subsection{Proof of Proposition~\ref{prop:comm_necessary} (Communication Necessity)}

\begin{proof}
The hospital environment reward decomposes as:
\begin{equation}\label{eq:hospital_reward_decomp}
R = \sum_i R_i^{\text{base}} - \rho_{\text{blind}} - \rho_{\text{drug}} - \rho_{\text{resource}},\qquad \rho_{\text{blind}} = \sum_{i,j}\rho_{\text{blind},i,j}.
\end{equation}

Consider the blind treatment penalty for one (agent, patient) pair:
\begin{equation}
\rho_{\text{blind},i,j} = 1.5 \cdot (1 - \mathbbm{1}[\sigma_i = \xi_j]) \cdot h_{j,\xi_j} \cdot a_{i,j}.
\end{equation}

\textbf{Step 1: Expected Per-Step Penalty Without Communication.}
Without communication, agent $i$ cannot observe $h_{j,\xi_j}$ when $\sigma_i \neq \xi_j$. Under a non-communicating policy $\boldsymbol{\pi}^{\text{no-comm}}$, by linearity of expectation,
\begin{equation}
\E^{\boldsymbol{\pi}^{\text{no-comm}}}[\rho_{\text{blind},i,j}] = 1.5 \cdot P(\sigma_i \neq \xi_j) \cdot \E[h_{j,\xi_j}] \cdot \E[a_{i,j}].
\end{equation}

\textbf{Step 2: With Communication.}
With optimal communication $\boldsymbol{\pi}^*$, agent $i$ receives a message from the matching specialist that encodes $h_{j,\xi_j}$, so it can avoid aggressive treatment when the hidden risk is high, yielding $\E^{\boldsymbol{\pi}^*}[\rho_{\text{blind},i,j}] \approx 0$.

\textbf{Step 3: Per-Step Gap and Discounted Aggregation.}
Summing over all $(i,j)$ pairs and condition types $c\in\{1,\ldots,C\}$ yields a per-step expected penalty gap of
\begin{equation}
\Delta_{\text{step}} \;\ge\; \sum_{c=1}^C 1.5\cdot P(\xi\neq\sigma)\cdot\E[h_c]\cdot \bar a,
\end{equation}
with $\bar a>0$ the average action intensity. The value-function gap is the discounted sum of per-step gaps over the infinite horizon, and therefore
\begin{equation}
V^{\boldsymbol{\pi}^*} - V^{\boldsymbol{\pi}^{\text{no-comm}}} \;\ge\; \frac{\Delta_{\text{step}}}{1-\gamma} \;=\; \Omega\!\left(\frac{1}{1-\gamma}\sum_{c=1}^C \E[h_c] \cdot P(\xi \neq \sigma)\right).
\end{equation}
This establishes the claimed lower bound; for fixed $\gamma$ the $1/(1-\gamma)$ factor is absorbed into the $\Omega(\cdot)$ in the proposition statement.
\end{proof}

\paragraph{Architecture Matching.}
Both methods use:
\begin{itemize}
    \item Identical per-agent network architecture (2-layer MLP with hidden dimension $d_h = 128$)
    \item Same output heads for local state and reward prediction
    \item Orthogonal weight initialization
    \item LeakyReLU activations (negative slope $= 0.01$)
\end{itemize}

The \emph{only} architectural difference is that SeqComm includes a communication module, that adds messages to each agent's input. This ensures that any performance difference is attributable solely to the value of communication.

\paragraph{Preventing Overfitting.}
To ensure fair comparison, we monitor validation metrics (true environment TD loss) and report results at consistent training iterations. The matched learning rate ($3 \times 10^{-5}$) and high CG damping ($0.1$) provide implicit regularization that prevents overfitting in both methods.

\section{SMAC experiment details}
\label{app:smac_details}

\begin{figure}[H]
\centering
\includegraphics[width=0.98\textwidth]{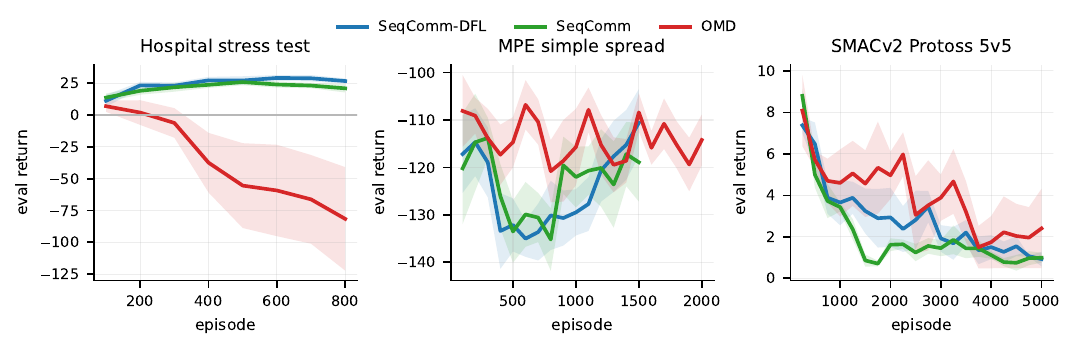}
\caption{Learning curves (seed means and shaded bands are standard errors over available matched seeds) over training. Hospital and MPE are completed 10-seed protocols; SMACv2 \texttt{protoss\_5\_vs\_5} is the completed clean-harness short-budget status curve and should be read together with Figure~\ref{fig:smacv2_status}.}
\label{fig:learning_curves}
\end{figure}

\subsection*{SMACv2 status}
\label{app:smacv2_deadline_status}

The main learning-curve figure, Figure~\ref{fig:learning_curves}, includes a short-budget SMACv2 status curve for \texttt{protoss\_5\_vs\_5}. The original SMAC table above reports the compact SMAC-v1 benchmark summary used in the main text; the additional table and figure below record the stricter SMACv2 runs with the official procedural generator~\citep{ellis2023smacv2}. These rows should be read as a budget-aware status snapshot rather than as a final budget-matched ranking, because the clean-harness communication jobs completed far fewer environment steps than the EPyMARL QMIX/MAPPO snapshots.

\paragraph{Why SMAC-v1 and SMACv2 separate.} Original SMAC uses fixed hand-designed scenarios, so an OMD-style predict-and-optimize update can repeatedly see the same unit compositions, observation geometry, and downstream optimizer. In that regime, a communication-conditioned world model can learn decision-relevant corrections and the value-aware auxiliary can improve the fixed QMIX-style decision map. SMACv2 deliberately changes the problem: scenarios are procedurally generated, evaluation requires generalization to unseen settings, and the benchmark was introduced because SMAC-v1 can be partially solved by policies with weak closed-loop dependence~\citep{ellis2023smacv2}. This makes the downstream optimizer less fixed from the perspective of DFL: the model-target mixture and value-aware gradients are computed under a rapidly changing distribution of local views, action masks, unit matchups, and teammate behaviors. The short-budget SMACv2 rows therefore identify a limitation of the present OMD/DFL instantiation rather than contradicting the original SMAC-v1 claim. They suggest that stronger procedural generalization will likely require adaptive model mixing, auxiliary gating based on model error or Q-spread, and longer matched budgets.

\begin{figure}[H]
\centering
\includegraphics[width=0.82\textwidth]{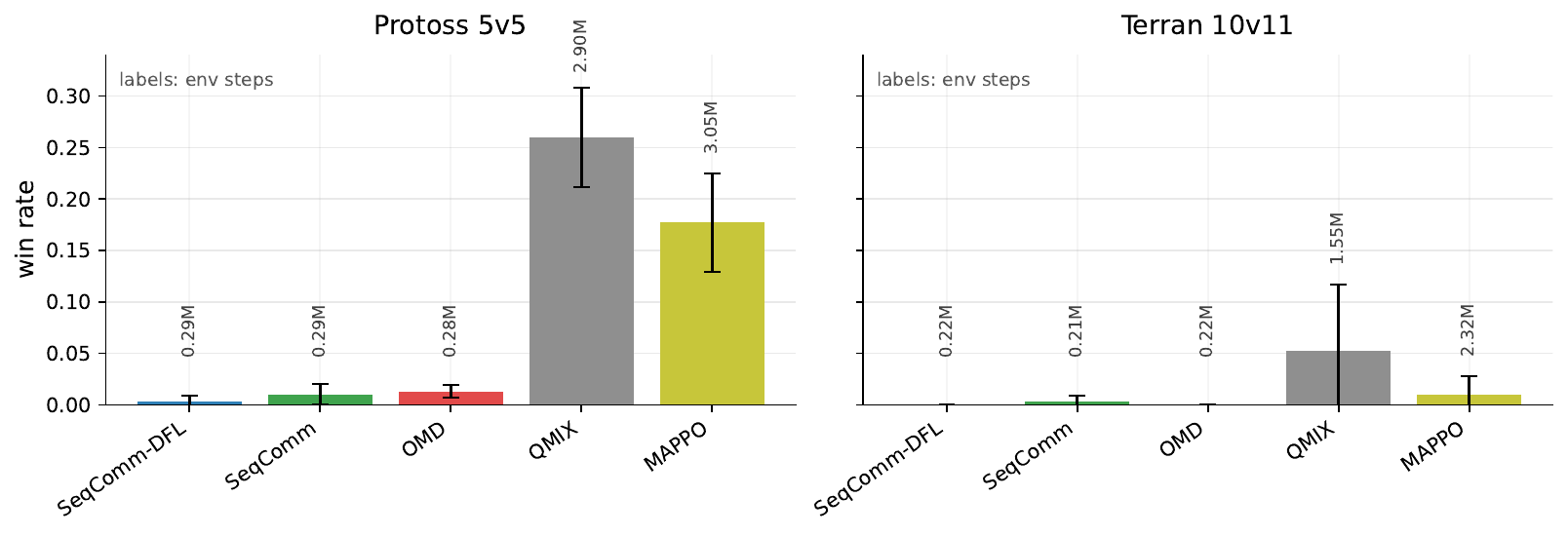}
\caption{SMACv2 deadline status by scenario. Bars report win-rate snapshots and vertical labels report completed or live environment-step budgets. The clean-harness communication rows completed far fewer steps than the EPyMARL QMIX/MAPPO snapshots, so this figure is a budget-aware status summary rather than a budget-matched ranking.}
\label{fig:smacv2_status}
\end{figure}

\begin{table}[H]
\centering
\caption{SMACv2 performance}
\label{tab:smacv2_deadline_results}
\resizebox{\textwidth}{!}{%
\begin{tabular}{llcccc}
\toprule
\textbf{Scenario} & \textbf{Variant} & \textbf{Seeds} & \textbf{Win rate} & \textbf{Return} & \textbf{Env steps} \\
\midrule
\texttt{protoss\_5\_vs\_5} & SeqComm-DFL  & 3 & $0.003\pm0.006$ & $7.20\pm0.67$ & $0.287$M \\
\texttt{protoss\_5\_vs\_5} & SeqComm  & 3 & $0.010\pm0.010$ & $8.42\pm0.32$ & $0.286$M \\
\texttt{protoss\_5\_vs\_5} & OMD & 3 & $0.013\pm0.006$ & $8.50\pm0.74$ & $0.285$M \\
\texttt{protoss\_5\_vs\_5} & QMIX & 3 & $0.260\pm0.048$ & $16.08\pm0.32$ & $2.902$M \\
\texttt{protoss\_5\_vs\_5} & MAPPO & 3 & $0.177\pm0.048$ & $14.02\pm0.93$ & $3.054$M \\
\texttt{terran\_10\_vs\_11} & SeqComm-DFL & 3 & $0.000\pm0.000$ & $5.11\pm0.59$ & $0.216$M \\
\texttt{terran\_10\_vs\_11} & SeqComm & 3 & $0.003\pm0.006$ & $5.18\pm0.39$ & $0.211$M \\
\texttt{terran\_10\_vs\_11} & OMD & 3 & $0.000\pm0.000$ & $5.76\pm0.69$ & $0.221$M \\
\texttt{terran\_10\_vs\_11} & QMIX & 3 & $0.052\pm0.065$ & $10.10\pm1.89$ & $1.551$M \\
\texttt{terran\_10\_vs\_11} & MAPPO & 3 & $0.010\pm0.018$ & $8.31\pm1.17$ & $2.319$M \\
\bottomrule
\end{tabular}}
\end{table}

\begin{table}[H]
\centering
\caption{SMACv2 SeqComm-DFL component ablation on \texttt{protoss\_5\_vs\_5}. These rows test the auxiliary pathways and are not part of the general baseline sweep.}
\label{tab:smacv2_focused_ablation}
\resizebox{\textwidth}{!}{%
\begin{tabular}{llcccc}
\toprule
\textbf{Scenario} & \textbf{Variant} & \textbf{Seeds} & \textbf{Win rate} & \textbf{Return} & \textbf{Env steps} \\
\midrule
\texttt{protoss\_5\_vs\_5} & SeqComm-DFL & 3 & $0.003\pm0.006$ & $7.20\pm0.67$ & $0.287$M \\
\texttt{protoss\_5\_vs\_5} & no value-aware & 3 & $0.003\pm0.006$ & $7.67\pm0.10$ & $0.304$M \\
\texttt{protoss\_5\_vs\_5} & no model mix & 3 & $0.003\pm0.006$ & $7.77\pm0.20$ & $0.307$M \\
\bottomrule
\end{tabular}}
\end{table}

\section{Ablation Study Details}
\label{app:ablation}

This appendix section supplies the detailed Hospital, MPE, and focused SMACv2 results that accompany the main experimental discussion in Section~\ref{sec:experiments} and the combined learning curves in Figure~\ref{fig:learning_curves}. The tables below keep the completed clean-harness evidence separate from the compact main-text narrative.

\paragraph{Result provenance.} The Hospital figure in the main text comes from the original OMD-vs.-SeqComm-DFL study and does not include the SeqComm baseline. The clean appendix runs use the updated stress-test dataset and add SeqComm as a communication baseline; these numbers should therefore be read as a newer harness comparison, not as a replacement for the original two-method figure. The severity-improvement panel in the main text belongs to that original two-method run. The clean 800-episode appendix logs do not record severity improvement, so Table~\ref{tab:hospital_main_clean} reports evaluation return only.

\subsection*{Clean Hospital stress-test results}

The Hospital learning-curve panel in Figure~\ref{fig:learning_curves} is an evaluation-return curve over the clean 800-episode stress-test runs. The clean result summary is included here to connect that curve to the appendix without changing the main text.

\begin{table}[H]
\centering
\caption{Hospital environment results. Random policy yields approximately $-96$. The lowmix-late row is the selected GPU stability recipe.}
\label{tab:hospital_main_clean}
\begin{tabular}{llccc}
\toprule
\textbf{Configuration} & \textbf{Run} & \textbf{Mean return} & \textbf{Std} & \textbf{Collapses} \\
\midrule
SeqComm-DFL (lowmix-late)      & GPU & $\phantom{-}29.3$ & $\phantom{-}4.1$ & $0/10$ \\
SeqComm-DFL (stable hybrid)    & local & $\phantom{-}25.5$ & $\phantom{-}7.5$ & $0/10$ \\
SeqComm baseline (standard comm.) & local & $\phantom{-}21.0$ & $10.3$ & $0/10$ \\
\gls{omd} baseline (no comm.)     & local & $-81.6$ & $128.6$ & $3/10$ \\
\bottomrule
\end{tabular}
\end{table}

\subsection*{MPE learning-curve and scaling results}

The MPE panel in Figure~\ref{fig:learning_curves} reports the completed 10-seed \texttt{simple\_spread} protocol. We additionally include the scaling sweep for 5- and 8-agent \texttt{simple\_spread} variants. Positive deltas in Table~\ref{tab:mpe_scaling} indicate settings where SeqComm-DFL improves over the strongest aggregate baseline for the same MPE configuration. The scaling result is conditional rather than universal: OMD is stronger in the 5-agent settings, while SeqComm-DFL is strongest in both 8-agent settings. This pattern suggests that the decision-focused communication auxiliaries are most useful when coordination scale increases the value of predecessor messages, while simpler baselines may remain preferable in lower-agent regimes.

\begin{table}[H]
\centering
\caption{Cross-environment validation: MPE \texttt{simple\_spread}, 10 seeds $\times$ 1500 episodes, GPU runs.}
\label{tab:mpe_main_clean}
\begin{tabular}{lcc}
\toprule
\textbf{Configuration} & \textbf{Mean $\pm$ std} & \textbf{Best seed} \\
\midrule
SeqComm-DFL (stable hybrid)    & $-110.6 \pm 22.4$ & $-83.2$ \\
\gls{omd} baseline (no comm.)     & $-114.2 \pm 17.0$ & $-91.0$ \\
SeqComm baseline (standard comm.) & $-118.9 \pm 26.2$ & $-88.5$ \\
\bottomrule
\end{tabular}
\end{table}

\begin{figure}[H]
\centering
\includegraphics[width=0.78\textwidth]{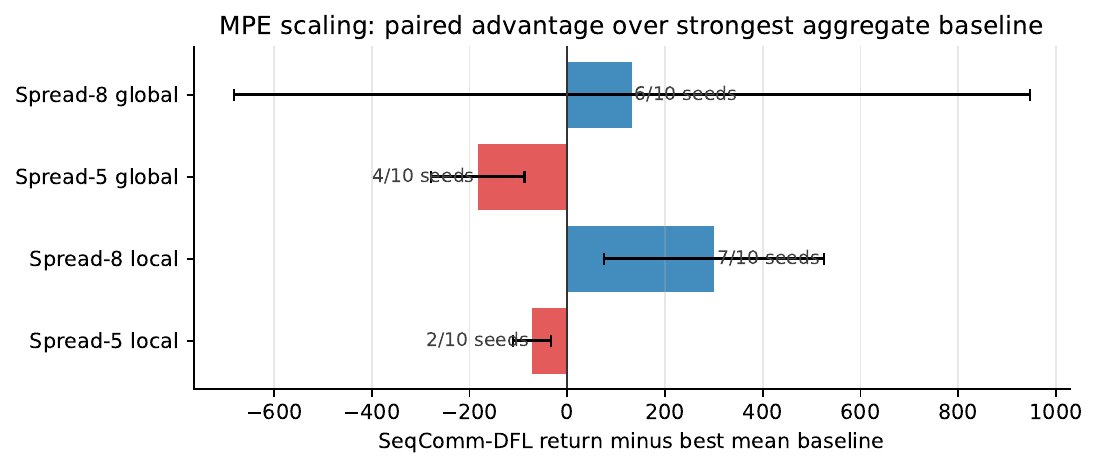}
\caption{MPE scaling paired deltas for the 5- and 8-agent sweep. Each bar is SeqComm-DFL final return minus the strongest aggregate baseline for the same setting, paired by seed; positive values favor SeqComm-DFL. The 8-agent rows favor SeqComm-DFL, while the 5-agent rows favor OMD, so the result supports a scale-dependent rather than universal MPE claim.}
\label{fig:mpe_scaling_delta}
\end{figure}

\begin{table}[H]
\centering
\caption{MPE \texttt{simple\_spread} scaling sweep. Higher return is better. Delta is SeqComm-DFL minus the best aggregate baseline in the same setting.}
\label{tab:mpe_scaling}
\resizebox{\textwidth}{!}{%
\begin{tabular}{lccc}
\toprule
\textbf{Setting} & \textbf{SeqComm-DFL} & \textbf{Best aggregate baseline} & \textbf{Delta} \\
\midrule
Spread-5 local ($N=5$, local ratio $0.5$) & $-445.9\pm96.7$ & OMD: $-373.9\pm74.4$ & $-71.9$ \\
Spread-8 local ($N=8$, local ratio $0.5$) & $-1044.5\pm311.9$ & OMD: $-1345.4\pm657.0$ & $+300.9$ \\
Spread-5 global ($N=5$, local ratio $0$) & $-737.8\pm267.6$ & OMD: $-554.5\pm147.8$ & $-183.3$ \\
Spread-8 global ($N=8$, local ratio $0$) & $-2735.7\pm1887.5$ & OMD: $-2867.9\pm1838.3$ & $+132.3$ \\
\bottomrule
\end{tabular}}
\end{table}

\subsection*{Original component ablations}

The following component table is the original Hospital ablation suite used for the component names in the main text: value-aware loss, Stackelberg conditioning, guidance potential, and counterfactual influence. It should not be conflated with the focused SMACv2 no-value-aware/no-model-mix rows above, which answer a narrower large-scale stability question.


\begin{table}[H]
\centering
\caption{Component ablation (Hospital environment, 3 agents, 10 seeds).}
\begin{tabular}{lcc}
\toprule
\textbf{Configuration} & \textbf{Episode Reward} & \textbf{$\Delta$ vs Full} \\
\midrule
Full SeqComm-DFL & $-48.3 \pm 5.2$ & --- \\
w/o Value-Aware ($\Lc_{\text{VA}}$) & $-54.1 \pm 6.1$ & $-12.0\%$ \\
w/o Stackelberg (parallel msgs) & $-52.7 \pm 5.8$ & $-9.1\%$ \\
w/o Guidance Potential & $-50.9 \pm 5.5$ & $-5.4\%$ \\
w/o Counterfactual Influence & $-49.8 \pm 5.4$ & $-3.1\%$ \\
\bottomrule
\end{tabular}
\end{table}

\begin{table}[H]
\centering
\caption{Communication dimension ablation.}
\begin{tabular}{lcc}
\toprule
$d_m$ & Hospital Reward & SMAC Win Rate \\
\midrule
4 & $-56.2 \pm 6.3$ & 0.69 $\pm$ 0.05 \\
8 (default) & $-48.3 \pm 5.2$ & 0.78 $\pm$ 0.04 \\
16 & $-47.8 \pm 5.4$ & 0.77 $\pm$ 0.05 \\
32 & $-49.1 \pm 5.7$ & 0.75 $\pm$ 0.05 \\
\bottomrule
\end{tabular}
\end{table}

\printglossaries




\newpage
\input{checklist.tex}

\end{document}

%% file: checklist.tex
\section*{NeurIPS Paper Checklist}

\begin{enumerate}

\item {\bf Claims}
    \item[] Question: Do the main claims made in the abstract and introduction accurately reflect the paper's contributions and scope?
    \item[] Answer: \answerYes{}
    \item[] Justification: The abstract and introduction claim (1) value-aware message generation with sequential Stackelberg conditioning, (2) information-theoretic bounds on communication value, (3) $O(1/\sqrt{T})$ bilevel convergence, and (4) 4--6$\times$ reward improvements and 13--15 pp win-rate gains on hospital and SMAC benchmarks. All four claims are substantiated: (1) is formalized in \S4.1--4.2 and Algorithms~1--2; (2) is proved in Theorem~5.1 and Appendix~D; (3) is proved in Theorem~5.2 and Appendix~D; and (4) is reported in Tables~2--3 and Figure~1 with 10-seed error bars.

\item {\bf Limitations}
    \item[] Question: Does the paper discuss the limitations of the work performed by the authors?
    \item[] Answer: \answerYes{}
    \item[] Justification: Limitations are addressed in two places. The Future Work paragraph of \S8 explicitly identifies scope boundaries: the current framework is restricted to discrete action spaces, fixed communication topology, and synthetic benchmarks; extensions to continuous actions, dynamic topology, and real clinical deployment are noted as open problems. Theoretical limitations are also embedded in the assumptions of Theorems~5.1--5.2 (Assumption~5.3 in Appendix~D): the convergence guarantee requires $L$-smooth losses and bounded gradients, and the $Q^*$ approximation bound in Proposition~5.1 applies to the tabular/linear regime.

\item {\bf Theory assumptions and proofs}
    \item[] Question: For each theoretical result, does the paper provide the full set of assumptions and a complete (and correct) proof?
    \item[] Answer: \answerYes{}
    \item[] Justification: All four formal results have complete proofs in the appendix. Theorem~5.1 (Communication Lower Bound) is proved in Appendix~D.1 using Pinsker's inequality and the definition of the coordination information gap (Definition~D.1). Proposition~5.1 ($Q^*$ approximation) is proved in Appendix~B.2 via the Bellman contraction argument. Theorem~5.2 (Convergence to Stationary Points) is proved in Appendix~D.3 via the descent lemma, bias decomposition, and telescoping; regularity conditions (Lipschitz smoothness, bounded gradients, Hessian invertibility) are collected in Assumption~5.3. Proposition~5.2 (Communication Necessity for the hospital environment) is proved in Appendix~I.3. All theorems and propositions in the main text are numbered and cross-referenced to the corresponding appendix proofs.

    \item {\bf Experimental result reproducibility}
    \item[] Question: Does the paper fully disclose all the information needed to reproduce the main experimental results of the paper to the extent that it affects the main claims and/or conclusions of the paper (regardless of whether the code and data are provided or not)?
    \item[] Answer: \answerYes{}
    \item[] Justification: All information necessary to reproduce experiments is provided. The hospital environment is fully specified in Appendix~I (state/observation/action spaces, reward function including blind-treatment, drug-interaction, and resource penalties, and transition dynamics). SMAC maps are standard and publicly available~\citep{samvelyan2019starcraft}. Full hyperparameters are listed in Appendix~H (Table~4) including learning rates, CG damping, discount factor, inner-loop iterations, communication dimension, and all loss weights. Algorithms~1--2 describe the training procedure step-by-step. All results use 10 random seeds with reported mean $\pm$ standard deviation.

\item {\bf Open access to data and code}
    \item[] Question: Does the paper provide open access to the data and code, with sufficient instructions to faithfully reproduce the main experimental results, as described in supplemental material?
    \item[] Answer: \answerNo{}
    \item[] Justification: Code is not released with the submission to preserve anonymity. However, the hospital environment is fully described in Appendix~I so it can be re-implemented, the SMAC benchmark is publicly available, and all hyperparameters and architectural details are disclosed in Appendix~H, making independent reproduction feasible. We plan to release the full code upon acceptance.

\item {\bf Experimental setting/details}
    \item[] Question: Does the paper specify all the training and test details (e.g., data splits, hyperparameters, how they were chosen, type of optimizer) necessary to understand the results?
    \item[] Answer: \answerYes{}
    \item[] Justification: Appendix~H provides a complete hyperparameter table (Table~4) covering all optimizer settings (Adam with world model LR $3\times10^{-5}$, critic LR $10^{-4}$), architecture details (2-layer MLP, hidden dimension 128, LeakyReLU activations, orthogonal initialization), communication dimension ($d_m=8$), all loss weights ($\lambda_{\text{VA}}=0.1$, $\lambda_{\text{inf}}=0.01$, $\lambda_{\text{aware}}=0.05$), bilevel parameters ($K_{\text{inner}}=15$, $K_{\text{CG}}=10$, $\lambda_{\text{CG}}=0.1$), and the MC annealing schedule. Hyperparameters were selected by grid search on the hospital environment validation set. Both environments are evaluated without a train/test split (online RL).

\item {\bf Experiment statistical significance}
    \item[] Question: Does the paper report error bars suitably and correctly defined or other appropriate information about the statistical significance of the experiments?
    \item[] Answer: \answerYes{}
    \item[] Justification: All quantitative results report mean $\pm$ one standard deviation computed over 10 independent random seeds, with different initialization and environment seeds per run. This applies to SMAC win rates (Table~2), hospital ablation rewards (Appendix~F Tables), and communication dimension ablation (Appendix~F). The number of seeds (10) is stated in \S6. Learning curve figures show individual run variance. The source of variability (random weight initialization and environment stochasticity across seeds) is described in \S6.

\item {\bf Experiments compute resources}
    \item[] Question: For each experiment, does the paper provide sufficient information on the computer resources (type of compute workers, memory, time of execution) needed to reproduce the experiments?
    \item[] Answer: \answerYes{}
    \item[] Justification: All experiments were run on NVIDIA A100 GPUs. SMAC experiments ran for 10M environment steps per run over 10 random seeds; the hospital environment is smaller in scale. The computational complexity analysis in Appendix~E provides operation counts (forward passes per step) and a $\sim25\times$ speedup comparison vs.\ the original SeqComm baseline, giving a relative efficiency estimate. Wall-clock runtimes per run were approximately 6--18 hours per experiment for SMAC maps and 1--2 days for the hospital environment on a single A100.
    
\item {\bf Code of ethics}
    \item[] Question: Does the research conducted in the paper conform, in every respect, with the NeurIPS Code of Ethics \url{https://neurips.cc/public/EthicsGuidelines}?
    \item[] Answer: \answerYes{}
    \item[] Justification: The work is foundational research on multi-agent reinforcement learning and communication. No human subjects, sensitive data, or proprietary datasets are involved. The hospital environment is a synthetic simulation with no real patient data. The paper is submitted anonymously per NeurIPS guidelines. No harmful applications are enabled by this research.

\item {\bf Broader impacts}
    \item[] Question: Does the paper discuss both potential positive societal impacts and negative societal impacts of the work performed?
    \item[] Answer: \answerYes{}
    \item[] Justification: The Broader Impact paragraph in \S8 discusses positive impacts (principled communication design for safety-critical systems such as healthcare coordination, where decision-quality-focused messaging can improve treatment outcomes and reduce adverse drug interactions). Potential negative impacts include dual-use in adversarial or surveillance contexts---e.g., optimizing communication in multi-agent systems for tracking or autonomous weapons. Because the contribution is a general algorithmic framework, these risks are indirect; however, they are acknowledged and mitigation through deployment oversight is implied by the ``safety-critical'' framing.
    
\item {\bf Safeguards}
    \item[] Question: Does the paper describe safeguards that have been put in place for responsible release of data or models that have a high risk for misuse (e.g., pre-trained language models, image generators, or scraped datasets)?
    \item[] Answer: \answerNA{}
    \item[] Justification: The paper releases no pre-trained models, scraped datasets, or high-risk generative assets. The hospital environment is a fully synthetic simulation with no real patient data, and the SMAC benchmark is an established public game environment with no personal or sensitive information.

\item {\bf Licenses for existing assets}
    \item[] Question: Are the creators or original owners of assets (e.g., code, data, models), used in the paper, properly credited and are the license and terms of use explicitly mentioned and properly respected?
    \item[] Answer: \answerYes{}
    \item[] Justification: All third-party assets are properly cited. The StarCraft Multi-Agent Challenge (SMAC) is credited to \citet{samvelyan2019starcraft} and is released under the MIT License. The OMD framework is attributed to \citet{nikishin2022control}. SeqComm is attributed to \citet{ding2024multilevel}. PyEPO is attributed to \citet{tang2024pyepo}. No assets are re-packaged or redistributed; they are used solely for comparison and methodological building blocks.

\item {\bf New assets}
    \item[] Question: Are new assets introduced in the paper well documented and is the documentation provided alongside the assets?
    \item[] Answer: \answerYes{}
    \item[] Justification: The paper introduces two new assets: (1) the collaborative hospital Dec-POMDP environment and (2) the SeqComm-DFL algorithm implementation. The hospital environment is fully documented in Appendix~I, including state/observation/action spaces, reward decomposition (blind-treatment, drug-interaction, and resource penalties), and transition dynamics. The SeqComm-DFL algorithm is documented via pseudocode in Algorithms~1--2, hyperparameter tables in Appendix~H, and architecture details in Appendix~I. Code will be released under an open-source license upon acceptance.

\item {\bf Crowdsourcing and research with human subjects}
    \item[] Question: For crowdsourcing experiments and research with human subjects, does the paper include the full text of instructions given to participants and screenshots, if applicable, as well as details about compensation (if any)? 
    \item[] Answer: \answerNA{}
    \item[] Justification: The paper involves no crowdsourcing and no research with human subjects. All experiments are conducted in simulation (synthetic hospital environment and SMAC game engine).

\item {\bf Institutional review board (IRB) approvals or equivalent for research with human subjects}
    \item[] Question: Does the paper describe potential risks incurred by study participants, whether such risks were disclosed to the subjects, and whether Institutional Review Board (IRB) approvals (or an equivalent approval/review based on the requirements of your country or institution) were obtained?
    \item[] Answer: \answerNA{}
    \item[] Justification: No human subjects are involved. The hospital environment is a fully synthetic simulation; no real patient data or clinical records are used at any stage of this research.

\item {\bf Declaration of LLM usage}
    \item[] Question: Does the paper describe the usage of LLMs if it is an important, original, or non-standard component of the core methods in this research? Note that if the LLM is used only for writing, editing, or formatting purposes and does \emph{not} impact the core methodology, scientific rigor, or originality of the research, declaration is not required.
    \item[] Answer: \answerNA{}
    \item[] Justification: LLMs are not part of the core methodology. SeqComm-DFL uses small MLPs for communication and value estimation in a reinforcement learning setting; no large language model is used as a component of the proposed method. LLMs were used only for grammar and writing assistance, which does not require declaration per NeurIPS policy.

\end{enumerate}